\documentclass{article}

\usepackage[nonatbib, preprint]{neurips_2020}

\usepackage{float}
\usepackage{bbm}
\usepackage{mathtools}
\usepackage{verbatim}
\usepackage{comment}
\usepackage{makecell}
\usepackage{microtype}
\usepackage{graphicx}
\usepackage{subfigure}
\usepackage{booktabs} 
\usepackage{adjustbox}
\usepackage{hyperref}

\usepackage[colorinlistoftodos]{todonotes}
\usepackage{verbatim}

\usepackage{tikz}

\usetikzlibrary{intersections,shapes.arrows}

\usepackage{amsthm,amsmath,amssymb, amsfonts}
\newtheorem{theorem}{Theorem}[section]
\newtheorem{lemma}[theorem]{Lemma}

\newtheorem{assumption}[theorem]{Assumption}

\usepackage{comment}
\usepackage{mathrsfs}
\usepackage{enumitem}
\usepackage{verbatim}

\newtheorem{thm}{Theorem}

\newtheorem{lem}{Lemma}

\usepackage{amsthm,amsmath,amssymb, amsfonts}

\usepackage{titlesec}
\titlespacing\section{0pt}{0pt plus 0pt minus 0pt}{0pt plus 0pt minus 0pt}
\titlespacing\subsection{0pt}{0pt plus 0pt minus 0pt}{0pt plus 0pt minus 0pt}
\titlespacing\subsubsection{0pt}{0pt plus 0pt minus 0pt}{0pt plus 0pt minus 0pt}

\usepackage[skip=0pt]{caption}
\allowdisplaybreaks

\title{An Ode to an ODE}

\author{%
  Krzysztof Choromanski \thanks{equal contribution}\\
  Google Brain \& Columbia University\\
   \And
  Jared Quincy Davis \footnotemark[1]\\
  DeepMind \& Stanford University\\
   \And
  Valerii Likhosherstov \footnotemark[1]\\
  University of Cambridge\\
   \And
  Xingyou Song\\
  Google Brain\\
   \And
  Jean-Jacques Slotine\\
Massachusetts Institute of Technology \\
   \And
  Jacob Varley\\
  Google Brain\\
  \And 
  Honglak Lee\\
  Google Brain\\
   \And 
  Adrian Weller\\
  University of Cambridge \& The Alan Turing Institute\\
   \And
  Vikas Sindhwani\\
  Google Brain\\
}

\begin{document}

\maketitle

\vspace{-0.4cm}
\begin{abstract}
\vspace{-0.2cm}
We present a new paradigm for Neural ODE algorithms, called $\mathrm{ODEtoODE}$, where time-dependent parameters of the main flow evolve according to a matrix flow on the orthogonal group $\mathcal{O}(d)$. This nested system of two flows, where the parameter-flow is constrained to lie on the compact manifold, provides stability and effectiveness of training and provably solves the gradient vanishing-explosion problem which is intrinsically related to training deep neural network architectures such as Neural ODEs. Consequently, it leads to better downstream models, as we show on the example of training reinforcement learning policies with evolution strategies, and in the supervised learning setting, by comparing with previous SOTA baselines. We provide strong convergence results for our proposed mechanism that are independent of the depth of the network, supporting our empirical studies. Our results show an intriguing connection between the theory of deep neural networks and the field of matrix flows on compact manifolds.
\end{abstract}

\section{Introduction}
Neural ODEs \cite{chen2018neural, chang2018reversible, haber2017stable} are natural continuous extensions of deep neural network architectures, with the evolution of the intermediate activations governed by an ODE:
\begin{equation} \frac{d\mathbf{x}_{t}}{dt} = f(\mathbf{x}_{t}, t, \theta) ,\label{eq:node}\end{equation} parameterized by $\theta\in {\mathbb{R}^n}$ and where $f:\mathbb{R}^{d} \times \mathbb{R} \times \mathbb{R}^{n} \rightarrow \mathbb{R}^{d}$ is some nonlinear mapping defining dynamics. A solution to the above system with initial condition $\mathbf{x}_{0}$ is of the form:
$$
\mathbf{x}_t = \mathbf{x}_0 + \int_{t_0}^t f(\mathbf{x}_{s}, s, \theta) ds, \label{eq:IVP_ODE}
$$ 
and can be approximated with various numerical integration techniques such as Runge-Kutta or Euler methods \cite{Sli2003AnIT}. The latter give rise to discretizations:
$$
\mathbf{x}_{t+dt} = \mathbf{x}_t + f(\mathbf{x}_t, t, \theta) dt, 
$$ 
that can be interpreted as discrete flows of ResNet-like computations \cite{HeZRS16} and establish a strong connection between the theory of deep neural networks and differential equations. This led to successful applications of Neural ODEs in machine learning. Those include in particular efficient time-continuous normalizing flows algorithms \cite{puchen} avoiding the computation of the determinant of the Jacobian (a computational bottleneck for discrete variants), as well as modeling latent dynamics in time-series analysis, particularly useful for handling irregularly sampled data \cite{time_series}.
Parameterized neural ODEs can be efficiently trained via adjoint sensitivity method \cite{chen2018neural} and are characterized by constant memory cost, independent of the depth of the system, with different parameterizations encoding different weight sharing patterns across infinitesimally close layers.

Such Neural ODE constructions 
enable deeper models 
than would not otherwise be possible with a fixed computation budget; however, it has been noted that training instabilities and the problem of vanishing/exploding gradients can arise during the learning of very deep systems \cite{pascanu, frasconi, deep_learning}.

To resolve these challenges for discrete recurrent neural network architectures, several improvements relying on transition transformations encoded by orthogonal/Hermitian matrices were proposed \cite{unitary, ort_0}. Orthogonal matrices, while coupled with certain classes of nonlinear mappings, provably preserve norms of the loss gradients during backpropagation through layers, yet they also incur potentially substantial Riemannian optimization costs \cite{edelman, stoch_1, hairer, absil}. Fortunately, there exist several efficient parameterizations of the subgroups of the orthogonal group $\mathcal{O}(d)$ that, even though in principle reduce representational capacity, in practice produce high-quality models and bypass Riemannian optimization
\cite{cwy, rahman, jing}.
All these advances address discrete settings and thus it is natural to ask what can be done for continuous systems, which by definition are deep.

In this paper, we answer this question by presenting a new paradigm for Neural ODE algorithms, called $\mathrm{ODEtoODE}$, where time-dependent parameters of the main flow evolve according to a matrix flow on the orthogonal group $\mathcal{O}(d)$. 
Such flows can be thought of as analogous to sequences of orthogonal matrices in discrete structured orthogonal models.
By linking the theory of training Neural ODEs with the rich mathematical field of matrix flows on compact manifolds, we can reformulate the problem of finding efficient Neural ODE algorithms as a task of constructing expressive parameterized flows on the orthogonal group. We show in this paper on the example of orthogonal flows corresponding to the so-called \textit{isospectral flows} \cite{bro, brockett1991dynamical, chu1988isospectral}, that such systems studied by mathematicians for centuries indeed help in training Neural ODE architectures (see: ISO-ODEtoODEs in Sec. \ref{sec:odetoode_types}).
There is a voluminous mathematical literature on using isospectral flows as continuous systems that solve a variety of combinatorial problems including sorting, matching and more \cite{bro, brockett1991dynamical}, but to the best of our knowledge, we are the first who propose to learn them as stabilizers for training deep Neural ODEs. 

Our proposed nested systems of flows, where the parameter-flow is constrained to lie on the compact manifold, provide stability and effectiveness of training, and provably solve the gradient vanishing/exploding problem for continuous systems. Consequently, they lead to better downstream models, as we show on a broad set of experiments (training reinforcement learning policies with evolution strategies and supervised learning). We support our empirical studies with strong convergence results, independent of the depth of the network. We are not the first to explore nested Neural ODE structure (see: \cite{anode_v2}). Our novelty is in showing that such hierarchical architectures can be significantly improved by an entanglement with flows on compact manifolds. 


To summarize, in this work we make the following contributions:

\begin{itemize}[wide, labelindent=0pt]
    \item We introduce new, explicit constructions of non-autonomous nested Neural ODEs ($\mathrm{ODEtoODEs}$) 
    where parameters are defined as rich functions of time evolving on compact matrix manifolds (Sec. \ref{sec:algorithm}). We present two main architectures: gated-$\mathrm{ODEtoODEs}$ and $\mathrm{ISO}$-$\mathrm{ODEtoODEs}$ (Sec. \ref{sec:odetoode_types}).
    \item We establish convergence results for $\mathrm{ODEtoODEs}$ on a broad class of Lipschitz-continuous objective functions, in particular in the challenging reinforcement learning setting (Sec. \ref{sec:rl_theory}).
    \item We then use the above framework to outperform previous Neural ODE variants and baseline architectures on RL tasks from $\mathrm{OpenAI}$ $\mathrm{Gym}$ and the $\mathrm{DeepMind}$ $\mathrm{Control}$ $\mathrm{Suite}$, and simultaneously to yield strong results on image classification tasks. To the best of our knowledge, we are the first to show that well-designed Neural ODE models with compact number of parameters make them good candidates for training reinforcement learning policies via evolutionary strategies (ES) \cite{structured}.
\end{itemize}
All proofs are given in the Appendix. We conclude in Sec. \ref{sec:conclusion} and discuss broader impact in Sec. \ref{broader_impact}.





\section{Related work}
Our work lies in the intersection of several fields: the theory of Lie groups, Riemannian geometry and deep neural systems. We provide an 
overview of the related literature below.

\textbf{Matrix flows.} Differential equations (DEs) on manifolds lie at the heart of modern differential geometry \cite{Ciarlet2005AnIT} which is a key ingredient to understanding structured optimization for stabilizing neural network training \cite{stoch_1}. In our work, we consider in particular matrix gradient flows on $\mathcal{O}(d)$ that are solutions to trainable matrix DEs. In order to efficiently compute these flows, we leverage the theory of compact Lie groups that are compact smooth manifolds equipped with rich algebraic structure \cite{lee, compact_lie, olver}. For efficient inference, we apply local linearizations of $\mathcal{O}(d)$ via its Lie algebra (skew-symmetric matrices vector spaces) and exponential mappings \cite{warner} (see: Sec. \ref{sec:algorithm}). 

\textbf{Hypernetworks.} An idea to use one neural network (hypernetwork) to provide weights for another network \cite{ha_1} can be elegantly adopted to the nested Neural ODE setting, where \cite{anode_v2} proposed to construct time-dependent parameters of one Neural ODE as a result of concurrently evolving another Neural ODE. While helpful in practice, this is insufficient to fully resolve the challenge of vanishing and exploding gradients. We expand upon this idea in our gated-$\mathrm{ODEtoODE}$ network, by constraining the hypernetwork to produce skew-symmetric matrices that are then translated to those from the orthogonal group $\mathcal{O}(d)$ via the aforementioned exponential mapping.

\textbf{Stable Neural ODEs.} On the line of work for stabilizing training of Neural ODEs, \cite{how_to_train_neural_ode} proposes regularization based on optimal transport, while \cite{stabilizing_neural_ode} adds Gaussian noise into the ODE equation, turning it into a stochastic dynamical system. \cite{augmented_neural_ode} lifts the ODE to a higher dimensional space to prevent trajectory intersections, while \cite{towards_robust_stochastic_differential_equations, stable_neural_flows} add additional terms into the ODE, inspired by well known physical dynamical systems such as Hamiltonian dynamics. 

\textbf{Orthogonal systems.} For classical non-linear networks, a broad class of methods which improve stability and generalization involve orthogonality constraints on the layer weights. Such methods range from merely orthogonal initialization \cite{orthogonal_init, orthogonal_init_proof} and orthogonal regularization \cite{BansalCW18, XieXP17} to methods which completely constrain weight matrices to the orthogonal manifold throughout training using Riemannian optimization \cite{lee}, by projecting the unconstrained gradient to the orthogonal group $\mathcal{O}(d)$ \cite{shalit}. Such methods have been highly useful in preventing exploding/vanishing gradients caused by long term dependencies in recurrent neural networks (RNNs) \cite{ortho_1, ortho_2}. \cite{dynamical_isometry_cnn, dynamical_isometry_rnn} note that \textit{dynamical isometry} can be preserved by orthogonality, which allows training networks of very large depth.

\section{Improving Neural ODEs via flows on compact matrix manifolds}
\label{sec:algorithm}
Our core $\mathrm{ODEtoODE}$ architecture is the following nested Neural ODE system: 
\begin{equation}
\label{main_eq}
\begin{cases} 
\dot{\mathbf{x}}_{t} = f(\mathbf{W}_{t}\mathbf{x}_{t}) \\ 
\dot{\mathbf{W}}_{t}  = \mathbf{W}_{t}b_{\psi}(t,\mathbf{W}_{t}) \\
\end{cases}
\end{equation}
for some function $f:\mathbb{R} \rightarrow \mathbb{R}$ (applied elementwise), and a parameterized function: $b_{\psi}:\mathbb{R} \times \mathbb{R}^{d \times d} \rightarrow \mathrm{Skew}(d)$, where $\mathrm{Skew}(d)$ stands for the vector space of skew-symmetric (antisymmetric) matrices in $\mathbb{R}^{d \times d}$. We take $\mathbf{W}_{0} \in \mathcal{O}(d)$, where the \textit{orthogonal group} $\mathcal{O}(d)$ is defined as: 
$\mathcal{O}(d) = \{\mathbf{M} \in \mathbb{R}^{d \times d}: \mathbf{M}^{\top}\mathbf{M}=\mathbf{I}_{d}\}$. 
\begin{figure}[t]
  \label{fig:benchmark}
  \centering
	\includegraphics[keepaspectratio, width=0.99\textwidth]{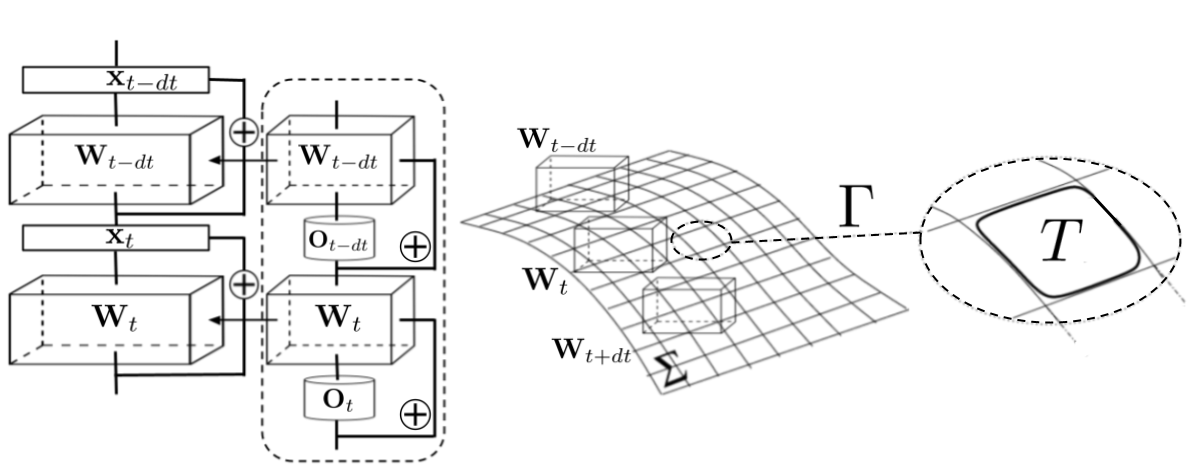}
  \caption{\small{Schematic representation of the discretized $\mathrm{ODEtoODE}$ architecture. On the left: nested ODE structure with parameters of the main flow being fed with the output of the matrix flow (inside dashed contour). On the right: the matrix flow evolves on the compact manifold $\Sigma$, locally linearized by vector space $T$ and with mapping $\Gamma$ encoding this linearization 
  ($\mathbf{O}_{t}=b_{\psi}(t,\mathbf{W}_{t})$).}}
\label{fig:scheme}
\end{figure}
It can be proven \cite{lee} that under these conditions $\mathbf{W}_{t} \in \mathcal{O}(d)$ for every $t \geq 0$. In principle $\mathrm{ODEtoODE}$ can exploit arbitrary compact matrix manifold $\Sigma$ (with $b_{\psi}$ modified respectively so that second equation in Formula \ref{main_eq} represents on-manifold flow), yet in practice we find that taking $\Sigma = \mathcal{O}(d)$ is the most effective. Furthermore, for $\Sigma=\mathcal{O}(d)$, the structure of $b_{\psi}$ is particularly simple, it is only required to output skew-symmetric matrices. Schematic representation of the $\mathrm{ODEtoODE}$ architecture is given in Fig. \ref{fig:scheme}. 
We take as $f$ a function that is non-differentiable in at most finite number of its inputs and such that $|f^{\prime}(x)|=1$ on all others (e.g. $f(x)=|x|$).

\subsection{Shaping ODEtoODEs}
\label{sec:odetoode_types}
An $\mathrm{ODEtoODE}$ is defined by: $\mathbf{x}_{0}$, $\mathbf{W}_{0}$ (initial conditions) and $b_{\psi}$. Vector $\mathbf{x}_{0}$ encodes input data e.g. state of an RL agent (Sec. \ref{sec:rl}) or an image (Sec. \ref{sec:supervised}). Initial point $\mathbf{W}_{0}$ of the matrix flow can be either learned (see: Sec. \ref{sec:learning}) or sampled upfront uniformly at random from Haar measure on $\mathcal{O}(d)$. 
We present two different parameterizations of $b_{\psi}$, leading to two different classes of $\mathrm{ODEtoODE}$s:
\paragraph{$\mathbf{ISO}$-$\mathbf{ODEtoODE}$s:} Those leverage a popular family of the \textit{isospectral flows} (i.e. flows preserving matrix spectrum), called \textit{double-bracket flows}, and given as: $
\mathbf{\dot{H}}_{t}  = [\mathbf{H}_{t}, [\mathbf{H}_{t},\mathbf{N}]],
$ where: $\mathbf{Q}\overset{\mathrm{def}}{=}\mathbf{H}_{0},\mathbf{N} \in \mathrm{Sym}(d) \subseteq \mathbb{R}^{d \times d}$, $\mathrm{Sym}(d)$ stands for the class of symmetric matrices in $\mathbb{R}^{d \times d}$ and $[]$ denotes the \textit{Lie bracket}: $[\mathbf{A},\mathbf{B}]\overset{\mathrm{def}}{=}\mathbf{AB}-\mathbf{BA}$. Double-bracket flows with customized parameter-matrices $\mathbf{Q},\mathbf{N}$ can be used to solve various combinatorial problems such as sorting, matching, etc. \cite{brockett1991dynamical}. It can be shown that $\mathbf{H}_{t}$ is similar to $\mathbf{H}_{0}$ for every $t \geq 0$, thus we can write: $\mathbf{H}_{t} = \mathbf{W}_{t}\mathbf{H}_{0}\mathbf{W}_{t}^{-1}$, where $\mathbf{W}_{t} \in \mathcal{O}(d)$. The corresponding flow on $\mathcal{O}(d)$ has the form: $\mathbf{\dot{W}}_{t} = \mathbf{W}_{t}[(\mathbf{W}_{t})^{\top} \mathbf{Q}\mathbf{W}_{t}, \mathbf{N}]$ and we take: $b_{\psi}^{\mathrm{iso}}(t,\mathbf{W}_{t})\overset{\mathrm{def}}{=}[(\mathbf{W}_{t})^{\top} \mathbf{Q}\mathbf{W}_{t}, \mathbf{N}]$, where $\psi=(\mathbf{Q},\mathbf{N})$ are learnable symmetric matrices.
It is easy to see that $b_{\psi}$ outputs skew-symmetric matrices.
\vspace{-3.5mm}
\paragraph{$\mathbf{Gated}$-$\mathbf{ODEtoODE}$s:} In this setting, inspired by \cite{chang2018reversible, anode_v2},  we simply take $b_{\psi}^{\mathrm{gated}} = \sum_{i=1}^{d} a_{i} B_{\psi}^{i}$, where $d$ stands for the number of gates, $a=(a_{1},...,a_{d})$ are learnable coefficients and $B^{i}_{\psi} \overset{\mathrm{def}}{=} f^{\psi^{i}}_{i} - (f^{\psi^{i}}_{i})^{\top}$. Here $\{f^{\psi^{i}}_{i}\}_{i=1,...,d}$ are outputs of neural networks parameterized by $\psi^{i}$s and producing unstructured matrices and $\psi=\mathrm{concat}(a,\psi^{1},...,\psi^{d})$. As above, $b_{\psi}^{\mathrm{gated}}$ outputs matrices in $\mathrm{Skew}(d)$.
\subsection{Learning and executing ODEtoODEs}
\label{sec:learning}
Note that most of the parameters of $\mathrm{ODEtoODE}$ from Formula \ref{main_eq} are unstructured and can be trained via standard methods. To train an initial orthogonal matrix $\mathbf{W}_{0}$, we apply tools from the theory of Lie groups and Lie algebras \cite{lee}. First, we notice that $\mathcal{O}(d)$ can be locally linearized via skew-symmetric matrices $\mathrm{Skew}(d)$ (i.e. a tangent vector space $T_{\mathbf{W}}$ to $\mathcal{O}(d)$ in $\mathbf{W} \in \mathcal{O}(d)$ is of the form: $T_{\mathbf{W}} = \mathbf{W}\cdot \mathrm{Sk}(d)$). We then compute Riemannian gradients which are projections of unstructured ones onto tangent spaces. For the inference on $\mathcal{O}(d)$, we apply exponential mapping $\Gamma(\mathbf{W}) \overset{\mathrm{def}}{=} \exp(\eta b_{\psi}(t,\mathbf{W}))$, where $\eta$ is the step size, leading to the discretization of the form: $\mathbf{W}_{t+dt}=\mathbf{W}_{t}\Gamma(\mathbf{W}_{t})$ (see: Fig. \ref{fig:scheme}).

\section{Theoretical results}
\subsection{ODEtoODEs solve gradient vanishing/explosion problem}
First we show that ODEtoODEs do not suffer from the gradient vanishing/explosion problem, i.e. the norms of loss gradients with respect to intermediate activations $\mathbf{x}_{t}$ do not grow/vanish exponentially while backpropagating through time (through \textit{infinitesimally close} layers).

\begin{lemma}[ODEtoODES for gradient stabilization]
\label{lemma:grad_vanish_explosion_lemma}
Consider a Neural ODE on time interval $[0, T]$ and given by Formula \ref{main_eq}. Let $\mathcal{L}=\mathcal{L}(\mathbf{x}_{T})$ be a  
differentiable loss function. The following holds for any $t \in [0, 1]$, where $e=2.71828...$ is Euler constant:
\begin{equation}
\frac{1}{e}\|\frac{\partial \mathcal{L}}{\partial \mathbf{x}_{T}}\|_{2} \leq \|\frac{\partial \mathcal{L}}{\partial \mathbf{x}_{t}}\|_{2} \leq e   
\|\frac{\partial \mathcal{L}}{\partial \mathbf{x}_{T}}\|_{2}.
\end{equation}
\end{lemma}
\subsection{ODEtoODEs for training reinforcement learning policies with ES}
\label{sec:rl_theory}
Here we show strong convergence guarantees for $\mathrm{ODEtoODE}$. For the clarity of the exposition, we consider $\mathrm{ISO}$-$\mathrm{ODEtoODE}$, even though similar results can be obtained for gated-$\mathrm{ODEtoODE}$. We focus on training RL policies with ES \cite{salimans}, which is mathematically more challenging to analyze than supervised setting, since it requires applying 
$\mathrm{ODEtoODE}$ several times throughout the rollout. Analogous results can be obtained for the conceptually simpler supervised setting.

Let $\mathrm{env}: \mathbb{R}^d \times \mathbb{R}^m \to \mathbb{R}^d \times \mathbb{R}$ be an "environment" function such that it gets current state $\mathbf{s}_k \in \mathbb{R}^d$ and action encoded as a vector $\mathbf{a}_k \in \mathbb{R}^m$ as its input and outputs the next state $\mathbf{s}_{k + 1}$ and the next score value $l_{k + 1} \in \mathbb{R}$: $(\mathbf{s}_{k + 1}, l_{k + 1}) = \mathrm{env} (\mathbf{s}_k, \mathbf{a}_k)$. We treat $\mathrm{env}$ as a ``black-box'' function, which means that we evaluate its value for any given input, but we don't have access to $\mathrm{env}$'s gradients. An overall score $L$ is a sum of all per-step scores $l_k$ for the whole rollout from state $\mathbf{s}_0$ to state $\mathbf{s}_K$:
\begin{equation} \label{eq:ldef}
    L(\mathbf{s}_0, \mathbf{a}_0, \dots, \mathbf{a}_{K - 1}) = \sum_{k = 1}^K l_k, \quad \forall k \in \{ 1, \dots, K \}: (\mathbf{s}_k, l_k) = \mathrm{env} (\mathbf{s}_{k - 1}, \mathbf{a}_{k - 1}).
\end{equation}
We assume that $\mathbf{s}_0$ is deterministic and known. The goal is to train a policy function $g_\theta: \mathbb{R}^d \to \mathbb{R}^m$ which, for the current state vector $\mathbf{s} \in \mathbb{R}^d$, returns an action vector $\mathbf{a} = g_\theta (\mathbf{s})$ where $\theta$ are trained parameters. By \textit{Stiefel manifold} $\mathcal{ST} (d_1, d_2)$ we denote a nonsquare extension of orthogonal matrices: if $d_1 \geq d_2$ then $\mathcal{ST} (d_1, d_2) = \{ \mathbf{\Omega} \in \mathbb{R}^{d_1 \times d_2} \, | \, \mathbf{\Omega}^\top \mathbf{\Omega} = I \}$, otherwise $\mathcal{ST} (d_1, d_2) = \{ \mathbf{\Omega} \in \mathbb{R}^{d_1 \times d_2} \, | \, \mathbf{\Omega} \mathbf{\Omega}^\top = I \}$. We define $g_\theta$ as an $\mathrm{ODEtoODE}$ with Euler discretization given below:
\begin{gather}
    g_\theta (\mathbf{s}) = \mathbf{\Omega}_2 \mathbf{x}_{N}, \quad \mathbf{x}_0 = \mathbf{\Omega}_1 \mathbf{s}, \quad \forall i \in \{ 1, \dots, N \}: \mathbf{x}_i = \mathbf{x}_{i - 1} + \frac{1}{N} f (\mathbf{W}_i \mathbf{x}_{i - 1} + \mathbf{b}) \label{eq:roll1} \\
    \mathbf{W}_i = \mathbf{G}_{i - 1} \exp \biggl( \frac{1}{N} ( \mathbf{W}_{i - 1}^\top \mathbf{Q} \mathbf{W}_{i - 1} \mathbf{N} - \mathbf{N} \mathbf{W}_{i - 1}^\top \mathbf{Q} \mathbf{W}_{i - 1} ) \biggr) \in \mathcal{O} (d), \label{eq:roll2} \\
    \theta = \{\mathbf{\Omega}_1 \in \mathcal{ST}(n, d), \mathbf{\Omega}_2 \in \mathcal{ST}(m, n), \mathbf{b} \in \mathbb{R}^n, \mathbf{N} \in \mathbb{S} (n), \mathbf{Q} \in \mathbb{S} (n), \mathbf{W}_0 \in \mathcal{O} (n) \} \in \mathbb{D} \nonumber
\end{gather}
where by $\mathbb{D}$ we denote $\theta$'s domain: $\mathbb{D} = \mathcal{ST} (n, d) \times \mathcal{ST} (m, n) \times \mathbb{R}^n \times \mathbb{S} (n) \times \mathbb{S} (n) \times \mathcal{O} (n)$. Define a final objective $F: \mathbb{D} \to \mathbb{R}$ to maximize as
\begin{equation} \label{eq:esrollout}
    F (\theta) = L (\mathbf{s}_0, \mathbf{a}_0, \dots, \mathbf{a}_{K - 1}), \quad \forall k \in \{ 1, \dots, K \}: \mathbf{a}_{k - 1} = g_\theta (\mathbf{s}_{k - 1}).
\end{equation}
For convenience, instead of $(\mathbf{s}_{out}, l) = \mathrm{env} (\mathbf{s}_{in}, \mathbf{a})$ we will write $\mathbf{s}_{out} = \mathrm{env}^{(1)} (\mathbf{s}_{in}, \mathbf{a})$ and $l = \mathrm{env}^{(2)} (\mathbf{s}_{in}, \mathbf{a})$. In our subsequent analysis we will use a not quite limiting Lipschitz-continuity assumption on $\mathrm{env}$ which intuitively means that a small perturbation of $\mathrm{env}$'s input leads to a small perturbation of its output. In addition to that, we assume that per-step loss $l_k = \mathrm{env}^{(2)} (\mathbf{s}_{k - 1}, \mathbf{a}_{k - 1})$ is uniformly bounded.
\begin{assumption} \label{as:lpsf}
There exist $M, L_1, L_2 > 0$ such that for any $\mathbf{s}', \mathbf{s}'' \in \mathbb{R}^d, \mathbf{a}', \mathbf{a}'' \in \mathbb{R}^m$ it holds:
\begin{gather*}
    | \mathrm{env}^{(2)}(\mathbf{s}', \mathbf{a}') | \leq M, \quad  \| \mathrm{env}^{(1)} (\mathbf{s}', \mathbf{a}') - \mathrm{env}^{(1)} (\mathbf{s}'', \mathbf{a}'') \|_2 \leq L_1 \delta, \\
    | \mathrm{env}^{(2)} (\mathbf{s}', \mathbf{a}') - \mathrm{env}^{(2)} (\mathbf{s}'', \mathbf{a}'') | \leq L_2 \delta, \quad \delta = \| \mathbf{s}' - \mathbf{s}'' \|_2 + \| \mathbf{a}' - \mathbf{a}'' \|_2 .
\end{gather*}
\end{assumption}

\begin{assumption} \label{as:lpssigma}
$f(\cdot)$ is Lipschitz-continuous with a Lipschitz constant $1$: $\forall x', x'' \in \mathbb{R}: | f(x') - f(x'') | \leq | x' - x'' |$. In addition to that, $f(0) = 0$.
\end{assumption}
Instead of optimizing $F (\theta)$ directly, we opt for optimization of a Gaussian-smoothed proxi $F_\sigma (\theta)$:
\begin{equation*}
    F_\sigma (\theta) = \mathbb{E}_{\epsilon \sim \mathcal{N} (0, I)} F (\theta + \sigma \epsilon)
\end{equation*}
where $\sigma > 0$. Gradient of $F_\sigma (\theta)$ has a form:
\begin{equation*}
    \nabla F_\sigma (\theta) = \frac{1}{\sigma} \mathbb{E}_{\epsilon \sim \mathcal{N} (0, I)} F (\theta + \sigma \epsilon) \epsilon,
\end{equation*}
which suggests an unbiased stochastic approximation such that it only requires to evaluate $F(\theta)$, i.e. no access to $F$'s gradient is needed, where $v$ is a chosen natural number:
\begin{equation*}
    \widetilde{\nabla} F_\sigma (\theta) = \frac{1}{\sigma v} \sum_{w = 1}^v F(\theta + \sigma \epsilon_w) \epsilon_w, \quad \epsilon_1, \dots, \epsilon_v \sim \text{i.i.d. } \mathcal{N} (0, I),
\end{equation*}

By defining a corresponding metric, $\mathbb{D}$ becomes a product of Riemannian manifolds and a Riemannian manifold itself. By using a standard Euclidean metric, we consider a Riemannian gradient of $F_\sigma (\theta)$:
\begin{gather}
    \nabla_\mathcal{R} F_\sigma (\theta) = \{ \nabla_{\mathbf{\Omega}_1} \mathbf{\Omega}_1^\top - \mathbf{\Omega}_1 \nabla_{\mathbf{\Omega}_1}^\top, \nabla_{\mathbf{\Omega}_2} \mathbf{\Omega}_2^\top - \mathbf{\Omega}_2 \nabla_{\mathbf{\Omega}_2}^\top, \nabla_\mathbf{b}, \frac{1}{2} (\nabla_\mathbf{N} + \nabla_\mathbf{N}^\top), \frac{1}{2} (\nabla_\mathbf{Q} + \nabla_\mathbf{Q}^\top), \label{eq:rg1} \\
    \nabla_{\mathbf{W}_0} \mathbf{W}_{0}^\top - \mathbf{W}_{0} \nabla_{\mathbf{W}_0}^\top  \}, \text{ where } \{ \nabla_{\mathbf{\Omega}_1}, \nabla_{\mathbf{\Omega}_2}, \nabla_\mathbf{b}, \nabla_\mathbf{N}, \nabla_\mathbf{Q}, \nabla_{\mathbf{W}_0} \} = \nabla F_\sigma (\theta). \label{eq:rg2}
\end{gather}
We use Stochastic Riemannian Gradient Descent \cite{srgd} to maximize $F(\theta)$. Stochastic Riemannian gradient estimate $\widetilde{\nabla}_\mathcal{R} F_\sigma (\theta)$ is obtained by substituting $\nabla \to \widetilde{\nabla}$ in (\ref{eq:rg1}-\ref{eq:rg2}). The following Theorem proves that SRGD is converging to a stationary point of the maximization objective $F(\theta)$ with rate $O(\tau^{-0.5 + \epsilon})$ for any $\epsilon > 0$. Moreover, the constant hidden in $O(\tau^{-0.5 + \epsilon})$ rate estimate doesn't depend on the length $N$ of the rollout.

\begin{thm} \label{th:conv}
Consider a sequence $\{ \theta^{(\tau)} = \{ \{ \mathbf{\Omega}_1^{(\tau)}, \mathbf{\Omega}_2^{(\tau)}, \mathbf{b}^{(\tau)}, \mathbf{N}^{(\tau)}, \mathbf{Q}^{(\tau)}, \mathbf{W}_{0}^{(\tau)} \} \in \mathbb{D} \}_{\tau = 0}^\infty$ where $\theta^{(0)}$ is deterministic and fixed and for each $\tau > 0$:
\begin{gather*}
    \mathbf{\Omega}_1^{(\tau)} = \exp (\alpha_\tau \widetilde{\nabla}_{\mathcal{R},\mathbf{\Omega}_1}^{(\tau)} ) \mathbf{\Omega}_1^{(\tau - 1)}, \quad \mathbf{\Omega}_2^{(\tau)} = \exp (\alpha_\tau \widetilde{\nabla}_{\mathcal{R},\mathbf{\Omega}_2}^{(\tau)} ) \mathbf{\Omega}_2^{(\tau - 1)}, \quad \mathbf{b}^{(\tau)} = \mathbf{b}^{(\tau - 1)} + \alpha_\tau \widetilde{\nabla}_{\mathcal{R},\mathbf{b}}^{(\tau)}, \\
    \mathbf{N}^{(\tau)} = \mathbf{N}^{(\tau - 1)} + \alpha_\tau \widetilde{\nabla}_{\mathcal{R},\mathbf{N}}^{(\tau)}, \quad \mathbf{Q}^{(\tau)} = \mathbf{Q}^{(\tau - 1)} + \alpha_\tau \widetilde{\nabla}_{\mathcal{R},\mathbf{Q}}^{(\tau)}, \quad \mathbf{W}_{0}^{(\tau)} = \exp (\alpha_\tau \widetilde{\nabla}_{\mathcal{R},\mathbf{W}_0}^{(\tau)}) \mathbf{W}_{0}^{(\tau - 1)}, \\
    \{ \widetilde{\nabla}_{\mathcal{R},\mathbf{\Omega}_1}^{(\tau)}, \widetilde{\nabla}_{\mathcal{R},\mathbf{\Omega}_2}^{(\tau)}, \widetilde{\nabla}_{\mathcal{R},\mathbf{b}}^{(\tau)}, \widetilde{\nabla}_{\mathcal{R},\mathbf{N}}^{(\tau)}, \widetilde{\nabla}_{\mathcal{R},\mathbf{Q}}^{(\tau)}, \widetilde{\nabla}_{\mathcal{R},\mathbf{W}_0}^{(\tau)} \} = \widetilde{\nabla}_\mathcal{R} F_\sigma (\theta^{(\tau)}), \quad \alpha_\tau = \tau^{-0.5}.
\end{gather*}
Then $\min_{0 \leq \tau' < \tau} \mathbb{E} [ \| \nabla_\mathcal{R} F_\sigma (\theta^{(\tau')}) \|_2^2 | \mathcal{F}_{\tau,D,D_b} ] \leq \mathcal{E} \cdot \tau^{-0.5 + \epsilon}$ for any $\epsilon > 0$ where $\mathcal{E}$ doesn't depend on $N$ and $D, D_b > 0$ are constants and $\mathcal{F}_{\tau,D,D_b}$ denotes a condition that for all $\tau' \leq \tau$ it holds $\| \mathbf{N}^{(\tau')} \|_2, \| \mathbf{Q}^{(\tau')} \|_2 < D, \| \mathbf{b}^{(\tau')} \|_2 < D_b$.
\end{thm}
\section{Experiments}
\label{sec:experiments}
We run two sets of experiments comparing $\mathrm{ODEtoODE}$ with several other methods in the supervised setting and to train RL policies with ES. To the best of our knowledge, we are the first to propose Neural ODE architectures for RL-policies and explore how the compactification of the number of parameters they provide can be leveraged by ES methods that benefit from compact models \cite{structured}.

\subsection{Neural ODE policies with $\mathrm{ODEtoODE}$ architectures}
\label{sec:rl}
\subsubsection{Basic setup}
In all Neural ODE methods we were integrating on time interval $[0,T]$ for $T=1$ and applied discretization with integration step size $\eta=0.04$
(in our ODEtoODE we used that $\eta$ for both: the main flow and the orthogonal flow on $\mathcal{O}(d)$). The dimensionality of the embedding of the input state $s$ was chosen to be $h=64$ for $\mathrm{OpenAI}$ $\mathrm{Gym}$ $\mathrm{Humanoid}$ (for all methods but $\mathrm{HyperNet}$, where we chose $h=16$, see: discussion below) and $h=16$ for all other tasks. 

Neural ODE policies were obtained by a linear projection of the input state into embedded space and Neural ODE flow in that space, followed by another linear projection into action-space.
In addition to learning the parameters of the Neural ODEs, we also trained their initial matrices and those linear projections. Purely linear policies were proven to get good rewards on $\mathrm{OpenAI}$ $\mathrm{Gym}$ environments \cite{ARS}, yet they lead to inefficient policies in practice \cite{stoch_1}; thus even those environments benefit from deep nonlinear policies. We enriched our studies with additional environments from $\mathrm{Deep}$ $\mathrm{Mind}$ $\mathrm{Control}$ $\mathrm{Suite}$.
We used standard deviation $\mathrm{stddev}=0.1$ of the Gaussian noise defining ES perturbations, ES-gradient step size $\delta=0.01$ and function $\sigma(x) = |x|$ as a nonlinear mapping. In all experiments we used $k=200$ perturbations per iteration \cite{structured}. 
\vspace{-0.3cm}
\paragraph{Number of policy parameters:} To avoid favoring ODEtoODEs, the other architectures were scaled in such a way that they has similar (but not smaller) number of parameters. The ablation studies were conducted for them across different sizes and those providing best (in terms of the final reward) mean curves were chosen. No ablation studies were run on ODEtoODEs. 
\subsubsection{Tested methods}
We compared the following algorithms including standard deep neural networks and Neural ODEs:
\vspace{-3mm}
\paragraph{ODEtoODE:} Matrices corresponding to linear projections were constrained to be taken from $\mathrm{Stiefel}$ $\mathrm{manifold}$ $\mathcal{ST}(d)$ \cite{stoch_1} (a generalization of the orthogonal group $\mathcal{O}(d)$). Evolution strategies (ES) were applied as follows to optimize entire policy. Gradients of Gaussian smoothings of the function: $F:\mathbb{R}^{m} \rightarrow \mathbb{R}$ mapping vectorized policy (we denote as $m$ the number of its parameters) to obtained reward were computed via standard Monte Carlo procedure \cite{structured}. 

For the part of the gradient vector corresponding to linear projection matrices and initial matrix of the flow, the corresponding Riemannian gradients were computed to make sure that their updates keep them on the respective manifolds \cite{stoch_1}. The Riemannian gradients were then used with the exact exponential mapping from $\mathrm{Skew}(d)$ to $\mathcal{O}(d)$. For unconstrained parameters defining orthogonal flow, standard ES-update procedure was used \cite{salimans}. We applied \textbf{ISO-ODEtoODE} version of the method (see: Sec. \ref{sec:odetoode_types}).

\vspace{-3mm}
\paragraph{Deep(Res)Net:} In this case we tested unstructured deep feedforward fully connected (ResNet) neural network policy with $t=25$ hidden layers. 

\vspace{-3mm}
\paragraph{BaseODE:} Standard Neural ODE $\frac{d\mathbf{x}(t)}{dt} = f(\mathbf{x}_{t})$, where $f$ was chosen to be a feedforward fully connected network with with two hidden layers of size $h$.
\newpage 
\vspace{-0.6cm}
\begin{figure}[H]
    \begin{minipage}{1.0\textwidth}
    \subfigure{\includegraphics[width=.49\linewidth]{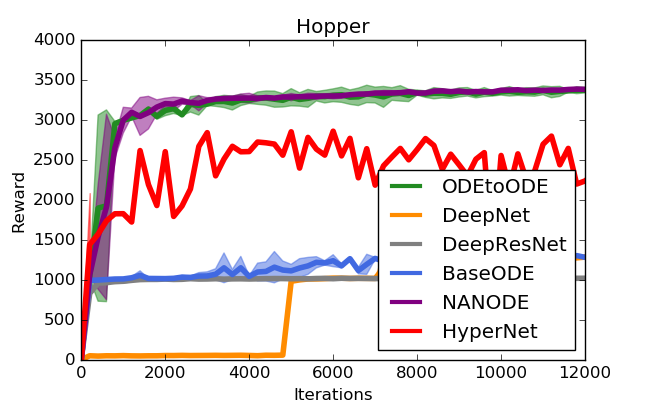}}
    \subfigure{\includegraphics[width=.49\linewidth]{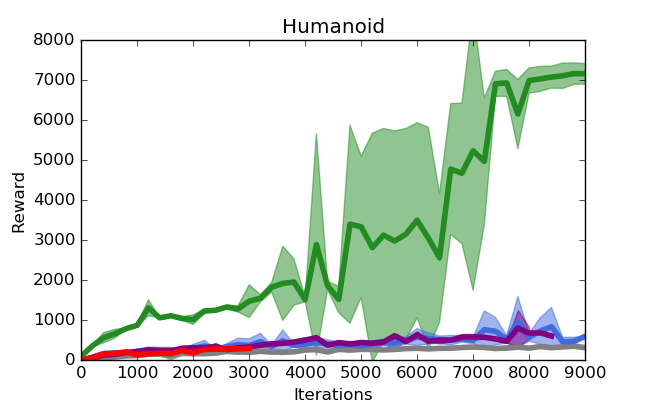}}  
\end{minipage}
\begin{minipage}{1.0\textwidth}
    \subfigure{\includegraphics[width=.49\linewidth]{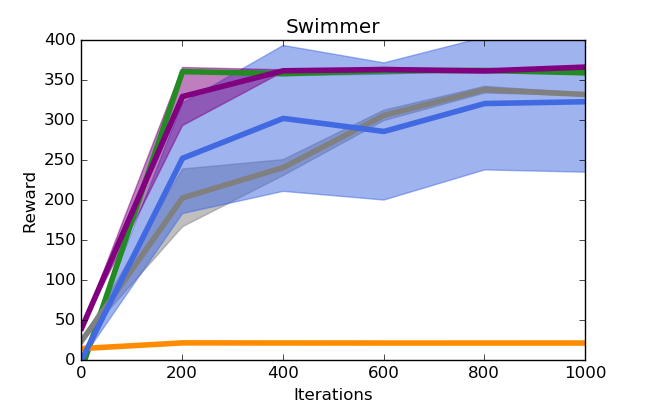}}
    \subfigure{\includegraphics[width=.49\linewidth]{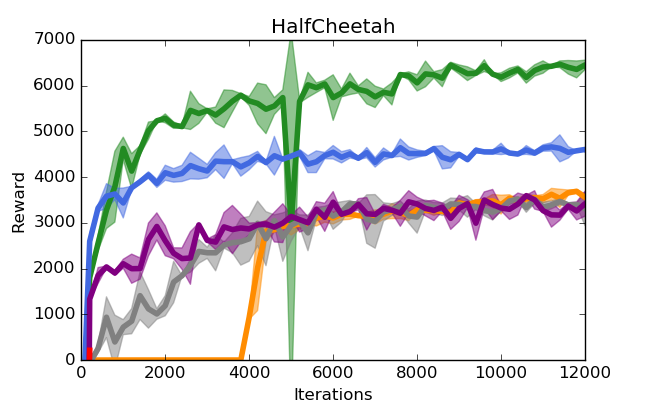}}  
\end{minipage}
\begin{minipage}{1.0\textwidth}
    \subfigure{\includegraphics[width=.49\linewidth]{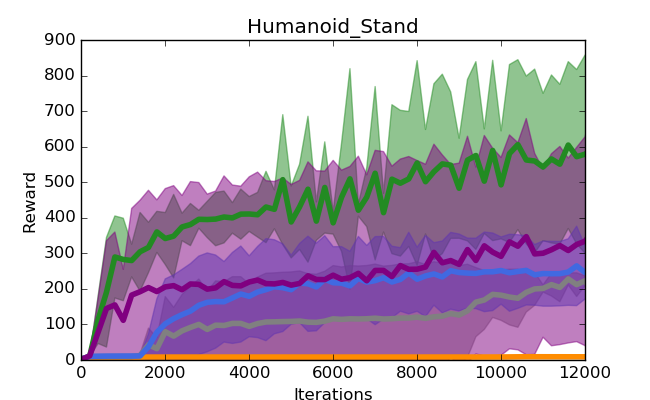}}
    \subfigure{\includegraphics[width=.49\linewidth]{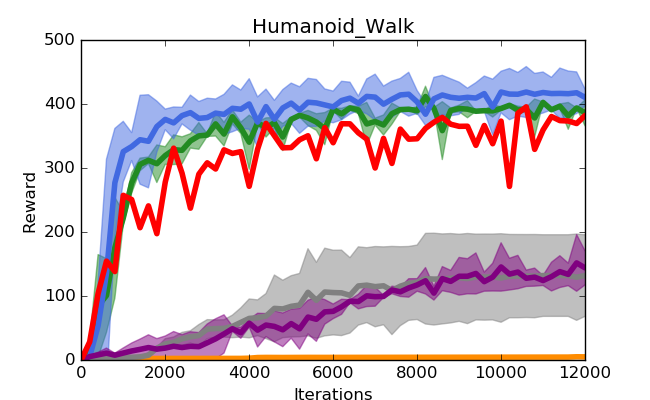}}  
\end{minipage}
\begin{minipage}{1.0\textwidth}
    \subfigure{\includegraphics[width=.49\linewidth]{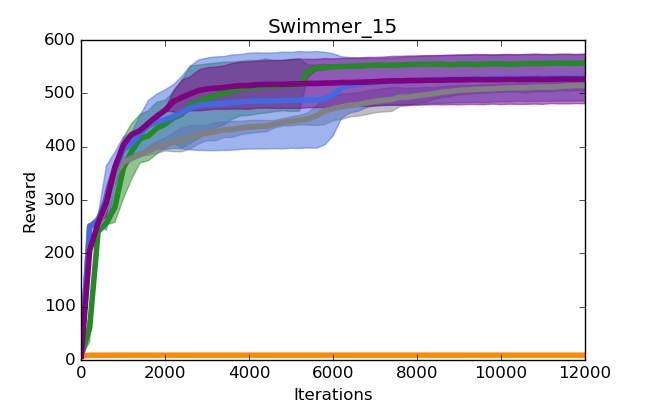}}
    \subfigure{\includegraphics[width=.49\linewidth]{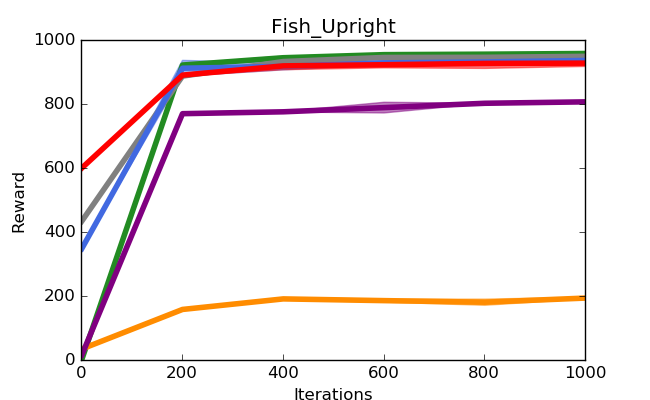}}  
\end{minipage}
\caption{\small{Comparison of different algorithms: $\mathrm{ODEtoODE}$, $\mathrm{DeepNet}$, $\mathrm{DeepResNet}$, $\mathrm{BaseODE}$, $\mathrm{NANODE}$ and $\mathrm{HyperNet}$ on four $\mathrm{OpenAI}$ $\mathrm{Gym}$ and four $\mathrm{DeepMind}$ $\mathrm{Control}$ $\mathrm{Suite}$ tasks. $\mathrm{HyperNet}$ on Swimmer, Swimmer\_15 and Humanoid\_Stand as well as $\mathrm{DeepNet}$ on Humanoid were not learning at all so corresponding curves were excluded. For $\mathrm{HalfCheetah}$ and $\mathrm{HyperNet}$, there was initially a learning signal (red spike at the plot near the origin) but then curve flattened at $0$.  Each plot shows $\mathrm{mean}$ +- $\mathrm{stdev}$ across $s=10$ seeds. $\mathrm{HyperNet}$ was also much slower than all other methods because of unstructured hypernetwork computations. $\mathrm{Humanoid}$ corresponds to $\mathrm{OpenAI}$ $\mathrm{Gym}$ environment and two its other versions on the plots to $\mathrm{DeepMind}$ $\mathrm{Control}$ $\mathrm{Suite}$ environment (with two different tasks).}}
\label{figure:es}
\end{figure}

\vspace{-9mm}
\paragraph{NANODE:} This leverages recently introduced class of Non-Autonomous Neural ODEs (NANODEs) \cite{nanode} that were showed to substantially outperform regular Neural ODEs in supervised training \cite{nanode}. NANODEs rely on flows of the form: $\frac{d\mathbf{x}(t)}{dt} = \sigma(\mathbf{W}_{t}\mathbf{x}_{t})$. Entries of weight matrices $\mathbf{W}_{t}$ are values of $d$-degree trigonometric polynomials \cite{nanode} with learnable coefficients. In all experiments we used $d=5$. We observed that was an optimal choice and larger values of $d$, even though improving model capacity, hurt training when number of perturbations was fixed (as mentioned above, in all experiments we used $k=200$).
\vspace{-3mm}
\paragraph{HyperNet:} For that method \cite{ha_1}, matrix $\mathbf{W}_{t}$ was obtained as a result of the entangled neural ODE in time $t$ after its output de-vectorization (to be more precise: the output of size $256$ was reshaped into a matrix from $\mathbb{R}^{16 \times 16}$). We encoded the dynamics of that entangled Neural ODE by a neural network $f$ with two hidden layers of size $s=16$ each. Note that even relatively small hypernetworks are characterized by large number of parameters which becomes a problem while training ES policies. We thus used $h=16$ for all tasks while running HyperNet algorithm. We did not put any structural assumptions on matrices $\mathbf{W}_{t}$, in particular they were not constrained to belong to $\mathcal{O}(d)$.
\subsubsection{Discussion of the results}
The results are presented in Fig. \ref{figure:es}. Our ODEtoODE is solely the best performing algorithm on four out of eight tasks: Humanoid, HalfCheetah, Humanoid\_Stand and Swimmer\_15 and is one of the two best performing on the remaining four.  It is clearly the most consistent algorithm across the board and the only one that learns good policies for Humanoid. Each of the remaining methods fails on at least one of the tasks, some such as: $\mathrm{HyperNet}$, $\mathrm{DeepNet}$ and $\mathrm{DepResNet}$ on more.
Exponential mapping (Sec. \ref{sec:learning}) computations for ODEtoODEs with hidden representations $h \leq 64$ took negligible time (we used well-optimized $\mathrm{scipy.linalg.expm}$ function). Thus all the algorithms but 
$\mathrm{HyperNet}$ (with expensive hypernetwork computations) had similar running time.

\subsection{Supervised learning with $\mathrm{ODEtoODE}$ architectures}
\label{sec:supervised}
We also show that $\mathrm{ODEtoODE}$ can be effectively applied in supervised learning by comparing it with multiple baselines on various image datasets.
\subsubsection{Basic setup}
All our supervised learning experiments use the $\mathrm{Corrupted}$ $\mathrm{MNIST}$ \cite{mu2019mnist} ($11$ different variants) dataset.
For all models in Table 1, we did not use dropout, applied hidden width $w=128$, and trained for $100$ epochs. For all models in Table 2, we used dropout with $r=0.1$ rate, hidden width $w=256$, and trained for $100$ epochs. For $\mathrm{ODEtoODE}$ variants, we used a discretization of $\eta=0.01$. 
\subsubsection{Tested methods}
We compared our $\mathrm{ODEtoODE}$ approach with several strong baselines: feedforward fully connected neural networks ($\mathrm{Dense}$), Neural ODEs inspired by \cite{chen2018neural} ($\mathrm{NODE}$), Non-Autonomous Neural ODEs \cite{nanode} ($\mathrm{NANODE}$) and hypernets \cite{anode_v2} ($\mathrm{HyperNet}$). 
For $\mathrm{NANODE}$, we vary the degree of the trigonometric polynomials. 
For $\mathrm{HyperNet}$, we used its gated version and compared against \textbf{gated-ODEtoODE} (see: Sec. \ref{sec:odetoode_types}). In this experiment, we focus on fully-connected architecture variants of all models. 
As discussed in prior work, orthogonality takes on a different form in the convolutional case \cite{wang2019orthogonal}, so we reserve discussion of this framing for future work.
\subsubsection{Discussion of the results}
\newcommand{\blue}[1]{\textcolor[rgb]{0.0,0.0,0.95}{#1}}
\newcommand{\red}[1]{\textcolor{red}{#1}}

The results are presented in Table \ref{table1} and Table \ref{table2} below. Our $\mathrm{ODEtoODE}$ outperforms other model variants on 11 out of the 15 tasks in Table \ref{table1}. On this particularly task, we find that $\mathrm{ODEtoODE}$-1, whereby we apply a constant perturbation to the hidden units $\theta{(0)}$ works best. We highlighted the best results for each task in \textbf{\blue{bolded blue}}.


\begin{table}[th!]
\centering\small
\resizebox{\textwidth}{!}{%
\begin{tabular}{@{}>{\bfseries}l*{10}{r}@{}}

\toprule[1pt]
\multicolumn{10}{@{}c@{}}{}\\
\midrule[0.75pt]

Models & \textbf{Dense-1} & \textbf{Dense-10} & \textbf{NODE}& \textbf{NANODE-1} & \textbf{NANODE-10} & \textbf{HyperNet-1} & \textbf{HyperNet-10} & \textbf{ODEtoODE-1} & \textbf{ODEtoODE-10} & \\
\midrule[0.55pt]

Dotted Lines & 92.99 & 88.22	& 91.54	& 92.42	& 92.74	& 91.88 & 91.9 & \blue{\textbf{95.42}} & 95.22 &  \\
Spatter & 93.54 & 89.52	& 92.49	& 93.13	& 93.15	& 93.19	& 93.28	& \blue{\textbf{94.9}}	& \blue{\textbf{94.9}}	& \\
Stripe & 30.55 & 20.57 & 36.76	& 16.4	& 21.37	& 19.71	& 18.69	& \blue{\textbf{44.51}} & 44.37	&  \\
Translate & 25.8 & 24.82 & 27.09	& 28.97	& 29.31	& \blue{\textbf{29.42}} & 29.3	& 26.82	& 26.63 & \\
Rotate & 82.8 & 80.38	& 82.76	& 83.19	& 83.65	& 83.5	& 83.56	& 83.9 & \blue{\textbf{84.1}} &  \\
Scale & 58.68 & 62.73	& 62.05	& 66.43	& 66.63	& 67.84	& \blue{\textbf{68.11}} & 66.68 & 66.76	&  \\
Shear & 91.73 & 89.52	& 91.82	& 92.18	& 93.11	& 92.33	& 92.48	& \blue{\textbf{93.37}} & 92.93	&  \\
Motion Blur & 78.16 & 67.25	& 75.18	& 76.53	& \blue{\textbf{79.22}}	& 78.82	& 78.33	& 78.63	& 77.58	&  \\
Glass Blur & 91.62 & 84.89	& 87.94	& 90.18	& 91.07	& 91.3	& 91.17	& \blue{\textbf{93.91}} & 93.29	&  \\
Shot Noise & 96.16 & 91.63 & 94.24 & 94.97 & 94.73 & 94.76 & 94.81 & \blue{\textbf{96.91}} & 96.71 &  \\
Identity & 97.55 & 95.73	& 97.61	& 97.65	& 97.69	& 97.72	& 97.54	& \blue{\textbf{97.94}} & 97.88	&  \\

\bottomrule[2pt]
\addlinespace[3pt]
\end{tabular}}
\caption{\small Test accuracy comparison of different methods. Postfix terms refer to hidden depth for Dense, trigonometric polynomial degree for NANODE, and number of gates for HyperNet and ODEtoODE.}\label{table1}
\end{table}

\vspace{-6mm}

\begin{table}[th!]
\centering\small
\resizebox{\textwidth}{!}{%
\begin{tabular}{@{}>{\bfseries}l*{14}{r}@{}}

\toprule[1pt]
\multicolumn{14}{@{}c@{}}{}\\
\midrule[0.75pt]

Models & \textbf{Dense-1} & \textbf{Dense-2} & \textbf{Dense-4} & \textbf{NODE} & \textbf{NANODE-1} & \textbf{NANODE-2} & \textbf{NANODE-4} & \textbf{HyperNet-1} & \textbf{HyperNet-2} & \textbf{HyperNet-4} & \textbf{ODEtoODE-1} & \textbf{ODEtoODE-2} & \textbf{ODEtoODE-4} & \\
\midrule[1.0pt]

Dotted Line	& 94.47 &	93.21 &	91.88 &	92.9 &	92 &	92.02 &	92.02 &	92.35 &	92.91 &	92.57 &	95.64 &	\blue{\textbf{95.66}} &	95.6  \\
Spatter	& 94.52	& 93.59	& 93.63 &	93.32 &	92.81 &	92.82 &	92.84 &	94.09	& 93.94	& 93.73	& 95.28	& \blue{\textbf{95.47}}	& 95.29	 \\
Stripe & 29.69	& 32.73	 & 31.49 &	36.08 &	27.32 &	23.56 &	24.66 &	30.86	& 31.12 &	29.1	& \blue{\textbf{36.25}}	& 28.21	& 31.91  \\
Translate	& 25.85	& 28.35	& 27.27	& 29.13	& 29.11 &	29.24 &	28.7	& 29.85 &	29.68 &	\blue{\textbf{29.87}} &	25.61 &	26.1 &	26.42  \\
Rotate & 83.62	& 83.73 & 83.88 &	83.09 &	82.5 &	82.77 & 83.03 & 83.96	& 84.04	& 84.13 &	\blue{\textbf{85.1}} &	85.03 &	84.72  \\
Scale & 61.41	& 65.07 & 63.72 &	63.62	& 65.49	& 65.33 &	64.03	& 69.51 &	68.77	& \blue{\textbf{69.8}}	& 67.97	& 66.95	& 67.09  \\
Shear & 92.25 &	92.76 &	92.55 &	92.27 &	92.08 &	92.18 &	92.3 &	93.35	& 92.84	& 93.04	& 93.15 &	93.06 &	\blue{\textbf{93.38}}  \\
Motion Blur	& 76.47 & 73.89 &	74.95 &	76.3 &	75.24	& 75.88	& 76.22	& 80.67 &	81.26 &	\blue{\textbf{81.36}} &	78.92 &	79.11 &	79.02  \\
Glass Blur & 92.5 &	89.65 &	90.29 &	89.47 &	89.5 &	89.5 &	89.8 &	93.07 &	92.67 &	92.69 &	\blue{\textbf{94.46}} &	94.19 &	94.3  \\
Shot Noise & 96.41 & 95.78 &	95.22 &	95.09 &	94.18  &	94.12 &	93.85 &	95.36 &	96.88 &	96.77 &	\blue{\textbf{96.93}} &	96.88 &	96.77  \\
Identity & 97.7 & 97.75 &	97.71 &	97.64 &	97.6 &	97.5 &	97.52 &	97.7	& 97.82	& 97.79 &	98.03 &	\blue{\textbf{98.12}} &	98.11  \\
\bottomrule[2pt]
\addlinespace[3pt]
\end{tabular}}
\caption{\small Additional Sweep for the setting as in Table \ref{table1}. This time all models also incorporate dropout with rate r = 0.1 and width w = 256. As above in Table \ref{table1}, $\mathrm{ODEtoODE}$ variants achieve the highest performance.}\label{table2}
\end{table}

\newpage
\section{Conclusion}
\label{sec:conclusion}
In this paper, we introduced nested Neural ODE systems, where the parameter-flow evolves on the orthogonal group $\mathcal{O}(d)$. Constraining this matrix-flow to develop on the compact manifold provides us with an architecture that can be efficiently trained without exploding/vanishing gradients, as we showed theoretically and demonstrated empirically by presenting gains on various downstream tasks. We are the first to demonstrate that algorithms training RL policies and relying on compact models' representations can significantly benefit from such hierarchical systems.

\section{Broader impact}
\label{broader_impact}
We do believe our contributions have potential broader impact that we briefly discuss below:
\vspace{-3mm}
\paragraph{Reinforcement Learning with Neural ODEs:} To the best of our knowledge, we are the first to propose to apply nested Neural ODEs in Reinforcement Learning, in particular to train Neural ODE policies. More compact architectures encoding deep neural network systems is an especially compelling feature in policy training algorithms, in particular while combined with ES methods admitting embarrassingly simple and efficient parallelization, yet suffering from high sampling complexity that increases with the number of policy parameters. Our work shows that RL training of such systems can be conducted efficiently provided that evolution of the parameters of the system is highly structured and takes place on compact matrix manifolds. Further extensions include applying ODEtoODEs in model-based reinforcement learning \cite{mbrl} to encode learnable dynamics of the system and in robotics. Machine learning for robotics is increasingly growing as a field and has potential of revolutionizing technology in the unprecedented way.

\vspace{-3mm}
\paragraph{Learnable Isospectral Flows:} We demonstrated that ISO-ODEtoODEs can be successfully applied to learn reinforcement learning policies. Those rely on the isospectral flows that in the past were demonstrated to be capable of solving combinatorial problems ranging from sorting to (graph) matching \cite{bro, zavlanos}. As emphasized before, such flows are however fixed and not trainable whereas we learn ours. That suggests that isospectral flows can be potentially learned to solve combinatorially-flavored machine learning problems or even integrated with non-combinatorial blocks in larger ML computational pipelines. The benefits of such an approach lie in the fact that we can efficiently backpropagate through these continuous systems and is related to recent research on differentiable sorting \cite{cuturi, blondel, grover}.
\bibliographystyle{abbrv}
\bibliography{references}

\begin{thebibliography}{10}

\bibitem{absil}
P.~Absil, R.~E. Mahony, and R.~Sepulchre.
\newblock {\em Optimization Algorithms on Matrix Manifolds}.
\newblock Princeton University Press, 2008.

\bibitem{unitary}
M.~Arjovsky, A.~Shah, and Y.~Bengio.
\newblock Unitary evolution recurrent neural networks.
\newblock In {\em Proceedings of the 33nd International Conference on Machine
  Learning, {ICML} 2016, New York City, NY, USA, June 19-24, 2016}, pages
  1120--1128, 2016.

\bibitem{BansalCW18}
N.~Bansal, X.~Chen, and Z.~Wang.
\newblock Can we gain more from orthogonality regularizations in training deep
  networks?
\newblock In {\em Advances in Neural Information Processing Systems 31: Annual
  Conference on Neural Information Processing Systems 2018, NeurIPS 2018, 3-8
  December 2018, Montr{\'{e}}al, Canada}, pages 4266--4276, 2018.

\bibitem{frasconi}
Y.~Bengio, P.~Frasconi, and P.~Y. Simard.
\newblock The problem of learning long-term dependencies in recurrent networks.
\newblock In {\em Proceedings of International Conference on Neural Networks
  (ICNN'88), San Francisco, CA, USA, March 28 - April 1, 1993}, pages
  1183--1188, 1993.

\bibitem{blondel}
M.~Blondel, O.~Teboul, Q.~Berthet, and J.~Djolonga.
\newblock Fast differentiable sorting and ranking.
\newblock {\em CoRR}, abs/2002.08871, 2020.

\bibitem{srgd}
S.~Bonnabel.
\newblock Stochastic gradient descent on riemannian manifolds.
\newblock {\em IEEE Transactions on Automatic Control}, 58(9):2217--2229, 2013.

\bibitem{compact_lie}
T.~Br{\"o}cker and T.~tom Dieck.
\newblock Representations of compact lie groups.
\newblock 1985.

\bibitem{bro}
R.~W. Brockett.
\newblock Least squares matching problems.
\newblock {\em Linear Algebra and its Applications}, 1989.

\bibitem{brockett1991dynamical}
R.~W. Brockett.
\newblock Dynamical systems that sort lists, diagonalize matrices, and solve
  linear programming problems.
\newblock {\em Linear Algebra and its applications}, 146:79--91, 1991.

\bibitem{chang2018reversible}
B.~Chang, L.~Meng, E.~Haber, L.~Ruthotto, D.~Begert, and E.~Holtham.
\newblock Reversible architectures for arbitrarily deep residual neural
  networks.
\newblock In {\em Thirty-Second AAAI Conference on Artificial Intelligence},
  2018.

\bibitem{puchen}
C.~Chen, C.~Li, L.~Chen, W.~Wang, Y.~Pu, and L.~Carin.
\newblock Continuous-time flows for efficient inference and density estimation.
\newblock In J.~G. Dy and A.~Krause, editors, {\em Proceedings of the 35th
  International Conference on Machine Learning, {ICML} 2018,
  Stockholmsm{\"{a}}ssan, Stockholm, Sweden, July 10-15, 2018}, volume~80 of
  {\em Proceedings of Machine Learning Research}, pages 823--832. {PMLR}, 2018.

\bibitem{dynamical_isometry_rnn}
M.~Chen, J.~Pennington, and S.~S. Schoenholz.
\newblock Dynamical isometry and a mean field theory of rnns: Gating enables
  signal propagation in recurrent neural networks.
\newblock In {\em Proceedings of the 35th International Conference on Machine
  Learning, {ICML} 2018, Stockholmsm{\"{a}}ssan, Stockholm, Sweden, July 10-15,
  2018}, pages 872--881, 2018.

\bibitem{chen2018neural}
T.~Q. Chen, Y.~Rubanova, J.~Bettencourt, and D.~K. Duvenaud.
\newblock Neural ordinary differential equations.
\newblock In {\em Advances in neural information processing systems}, pages
  6571--6583, 2018.

\bibitem{stoch_1}
K.~Choromanski, D.~Cheikhi, J.~Davis, V.~Likhosherstov, A.~Nazaret, A.~Bahamou,
  X.~Song, M.~Akarte, J.~Parker{-}Holder, J.~Bergquist, Y.~Gao, A.~Pacchiano,
  T.~Sarl{\'{o}}s, A.~Weller, and V.~Sindhwani.
\newblock Stochastic flows and geometric optimization on the orthogonal group.
\newblock {\em CoRR}, abs/2003.13563, 2020.

\bibitem{structured}
K.~Choromanski, M.~Rowland, V.~Sindhwani, R.~E. Turner, and A.~Weller.
\newblock Structured evolution with compact architectures for scalable policy
  optimization.
\newblock In {\em Proceedings of the 35th International Conference on Machine
  Learning, {ICML} 2018, Stockholmsm{\"{a}}ssan, Stockholm, Sweden, July 10-15,
  2018}, pages 969--977, 2018.

\bibitem{chu1988isospectral}
M.~T. Chu and L.~K. Norris.
\newblock Isospectral flows and abstract matrix factorizations.
\newblock {\em SIAM journal on numerical analysis}, 25(6):1383--1391, 1988.

\bibitem{Ciarlet2005AnIT}
P.~G. Ciarlet.
\newblock An introduction to differential geometry with applications to
  elasticity.
\newblock {\em Journal of Elasticity}, 78-79:1--215, 2005.

\bibitem{cuturi}
M.~Cuturi, O.~Teboul, and J.~Vert.
\newblock Differentiable ranking and sorting using optimal transport.
\newblock In H.~M. Wallach, H.~Larochelle, A.~Beygelzimer,
  F.~d'Alch{\'{e}}{-}Buc, E.~B. Fox, and R.~Garnett, editors, {\em Advances in
  Neural Information Processing Systems 32: Annual Conference on Neural
  Information Processing Systems 2019, NeurIPS 2019, 8-14 December 2019,
  Vancouver, BC, Canada}, pages 6858--6868, 2019.

\bibitem{nanode}
J.~Q. Davis, K.~Choromanski, J.~Varley, H.~Lee, J.~E. Slotine, V.~Likhosterov,
  A.~Weller, A.~Makadia, and V.~Sindhwani.
\newblock Time dependence in non-autonomous neural odes.
\newblock {\em CoRR}, abs/2005.01906, 2020.

\bibitem{augmented_neural_ode}
E.~Dupont, A.~Doucet, and Y.~W. Teh.
\newblock Augmented neural odes.
\newblock In {\em Advances in Neural Information Processing Systems 32: Annual
  Conference on Neural Information Processing Systems 2019, NeurIPS 2019, 8-14
  December 2019, Vancouver, BC, Canada}, pages 3134--3144, 2019.

\bibitem{edelman}
A.~Edelman, T.~A. Arias, and S.~T. Smith.
\newblock The geometry of algorithms with orthogonality constraints.
\newblock {\em {SIAM} J. Matrix Analysis Applications}, 20(2):303--353, 1998.

\bibitem{how_to_train_neural_ode}
C.~Finlay, J.~Jacobsen, L.~Nurbekyan, and A.~M. Oberman.
\newblock How to train your neural {ODE}.
\newblock {\em CoRR}, abs/2002.02798, 2020.

\bibitem{deep_learning}
I.~J. Goodfellow, Y.~Bengio, and A.~C. Courville.
\newblock {\em Deep Learning}.
\newblock Adaptive computation and machine learning. {MIT} Press, 2016.

\bibitem{grover}
A.~Grover, E.~Wang, A.~Zweig, and S.~Ermon.
\newblock Stochastic optimization of sorting networks via continuous
  relaxations.
\newblock In {\em 7th International Conference on Learning Representations,
  {ICLR} 2019, New Orleans, LA, USA, May 6-9, 2019}. OpenReview.net, 2019.

\bibitem{towards_robust_stochastic_differential_equations}
B.~G{\"{u}}ler, A.~Laignelet, and P.~Parpas.
\newblock Towards robust and stable deep learning algorithms for forward
  backward stochastic differential equations.
\newblock {\em CoRR}, abs/1910.11623, 2019.

\bibitem{ha_1}
D.~Ha, A.~M. Dai, and Q.~V. Le.
\newblock Hypernetworks.
\newblock {\em CoRR}, abs/1609.09106, 2016.

\bibitem{haber2017stable}
E.~Haber and L.~Ruthotto.
\newblock Stable architectures for deep neural networks.
\newblock {\em Inverse Problems}, 34(1):014004, 2017.

\bibitem{hairer}
E.~Hairer.
\newblock Important aspects of geometric numerical integration.
\newblock {\em J. Sci. Comput.}, 25(1):67--81, 2005.

\bibitem{HeZRS16}
K.~He, X.~Zhang, S.~Ren, and J.~Sun.
\newblock Deep residual learning for image recognition.
\newblock In {\em 2016 {IEEE} Conference on Computer Vision and Pattern
  Recognition, {CVPR} 2016, Las Vegas, NV, USA, June 27-30, 2016}, pages
  770--778, 2016.

\bibitem{ortho_1}
K.~Helfrich, D.~Willmott, and Q.~Ye.
\newblock Orthogonal recurrent neural networks with scaled {C}ayley transform.
\newblock In {\em Proceedings of the 35th International Conference on Machine
  Learning, {ICML} 2018, Stockholmsm{\"{a}}ssan, Stockholm, Sweden, July 10-15,
  2018}, pages 1974--1983, 2018.

\bibitem{ortho_2}
M.~Henaff, A.~Szlam, and Y.~LeCun.
\newblock Recurrent orthogonal networks and long-memory tasks.
\newblock In {\em Proceedings of the 33nd International Conference on Machine
  Learning, {ICML} 2016, New York City, NY, USA, June 19-24, 2016}, pages
  2034--2042, 2016.

\bibitem{orthogonal_init_proof}
W.~Hu, L.~Xiao, and J.~Pennington.
\newblock Provable benefit of orthogonal initialization in optimizing deep
  linear networks.
\newblock In {\em 8th International Conference on Learning Representations,
  {ICLR} 2020, Addis Ababa, Ethiopia, April 26-30, 2020}, 2020.

\bibitem{ort_0}
K.~Jia, S.~Li, Y.~Wen, T.~Liu, and D.~Tao.
\newblock Orthogonal deep neural networks.
\newblock {\em CoRR}, abs/1905.05929, 2019.

\bibitem{jing}
L.~Jing, Y.~Shen, T.~Dubcek, J.~Peurifoy, S.~A. Skirlo, Y.~LeCun, M.~Tegmark,
  and M.~Soljacic.
\newblock Tunable efficient unitary neural networks {(EUNN)} and their
  application to rnns.
\newblock In {\em Proceedings of the 34th International Conference on Machine
  Learning, {ICML} 2017, Sydney, NSW, Australia, 6-11 August 2017}, pages
  1733--1741, 2017.

\bibitem{lee}
J.~Lee.
\newblock {\em Introduction to smooth manifolds. 2nd revised ed}, volume 218.
\newblock 01 2012.

\bibitem{cwy}
V.~Likhosherstov, J.~Davis, K.~Choromanski, and A.~Weller.
\newblock {CWY} parametrization for scalable learning of orthogonal and stiefel
  matrices.
\newblock {\em CoRR}, abs/2004.08675, 2020.

\bibitem{stabilizing_neural_ode}
X.~Liu, T.~Xiao, S.~Si, Q.~Cao, S.~Kumar, and C.~Hsieh.
\newblock Neural {SDE:} stabilizing neural {ODE} networks with stochastic
  noise.
\newblock {\em CoRR}, abs/1906.02355, 2019.

\bibitem{ARS}
H.~Mania, A.~Guy, and B.~Recht.
\newblock Simple random search of static linear policies is competitive for
  reinforcement learning.
\newblock In {\em Advances in Neural Information Processing Systems 31: Annual
  Conference on Neural Information Processing Systems 2018, NeurIPS 2018, 3-8
  December 2018, Montr{\'{e}}al, Canada}, pages 1805--1814, 2018.

\bibitem{stable_neural_flows}
S.~Massaroli, M.~Poli, M.~Bin, J.~Park, A.~Yamashita, and H.~Asama.
\newblock Stable neural flows.
\newblock {\em CoRR}, abs/2003.08063, 2020.

\bibitem{rahman}
Z.~Mhammedi, A.~D. Hellicar, A.~Rahman, and J.~Bailey.
\newblock Efficient orthogonal parametrisation of recurrent neural networks
  using {Householder} reflections.
\newblock In {\em Proceedings of the 34th International Conference on Machine
  Learning, {ICML} 2017, Sydney, NSW, Australia, 6-11 August 2017}, pages
  2401--2409, 2017.

\bibitem{mu2019mnist}
N.~Mu and J.~Gilmer.
\newblock Mnist-c: A robustness benchmark for computer vision.
\newblock {\em arXiv preprint arXiv:1906.02337}, 2019.

\bibitem{olver}
P.~J. Olver.
\newblock {\em Applications of {L}ie Groups to Differential Equations}.
\newblock Springer, second edition, 2000.

\bibitem{pascanu}
R.~Pascanu, T.~Mikolov, and Y.~Bengio.
\newblock On the difficulty of training recurrent neural networks.
\newblock In {\em Proceedings of the 30th International Conference on Machine
  Learning, {ICML} 2013, Atlanta, GA, USA, 16-21 June 2013}, volume~28 of {\em
  {JMLR} Workshop and Conference Proceedings}, pages 1310--1318. JMLR.org,
  2013.

\bibitem{time_series}
Y.~Rubanova, T.~Q. Chen, and D.~Duvenaud.
\newblock Latent ordinary differential equations for irregularly-sampled time
  series.
\newblock In H.~M. Wallach, H.~Larochelle, A.~Beygelzimer,
  F.~d'Alch{\'{e}}{-}Buc, E.~B. Fox, and R.~Garnett, editors, {\em Advances in
  Neural Information Processing Systems 32: Annual Conference on Neural
  Information Processing Systems 2019, NeurIPS 2019, 8-14 December 2019,
  Vancouver, BC, Canada}, pages 5321--5331, 2019.

\bibitem{salimans}
T.~Salimans, J.~Ho, X.~Chen, and I.~Sutskever.
\newblock Evolution strategies as a scalable alternative to reinforcement
  learning.
\newblock {\em CoRR}, abs/1703.03864, 2017.

\bibitem{orthogonal_init}
A.~M. Saxe, J.~L. McClelland, and S.~Ganguli.
\newblock Exact solutions to the nonlinear dynamics of learning in deep linear
  neural networks.
\newblock In {\em 2nd International Conference on Learning Representations,
  {ICLR} 2014, Banff, AB, Canada, April 14-16, 2014, Conference Track
  Proceedings}, 2014.

\bibitem{shalit}
U.~Shalit and G.~Chechik.
\newblock Coordinate-descent for learning orthogonal matrices through {Givens}
  rotations.
\newblock In {\em Proceedings of the 31th International Conference on Machine
  Learning, {ICML} 2014, Beijing, China, 21-26 June 2014}, pages 548--556,
  2014.

\bibitem{Sli2003AnIT}
E.~S{\"u}li and D.~F. Mayers.
\newblock An introduction to numerical analysis.
\newblock 2003.

\bibitem{wang2019orthogonal}
J.~Wang, Y.~Chen, R.~Chakraborty, and S.~X. Yu.
\newblock Orthogonal convolutional neural networks.
\newblock {\em arXiv preprint arXiv:1911.12207}, 2019.

\bibitem{mbrl}
T.~Wang, X.~Bao, I.~Clavera, J.~Hoang, Y.~Wen, E.~Langlois, S.~Zhang, G.~Zhang,
  P.~Abbeel, and J.~Ba.
\newblock Benchmarking model-based reinforcement learning.
\newblock {\em CoRR}, abs/1907.02057, 2019.

\bibitem{warner}
F.~W. Warner.
\newblock Foundations of differentiable manifolds and {L}ie groups.
\newblock 1971.

\bibitem{wilcox}
R.~M. Wilcox.
\newblock Exponential operators and parameter differentiation in quantum
  physics.
\newblock {\em Journal of Mathematical Physics}, 8(4):962--982, 1967.

\bibitem{dynamical_isometry_cnn}
L.~Xiao, Y.~Bahri, J.~Sohl{-}Dickstein, S.~S. Schoenholz, and J.~Pennington.
\newblock Dynamical isometry and a mean field theory of cnns: How to train 10,
  000-layer vanilla convolutional neural networks.
\newblock In {\em Proceedings of the 35th International Conference on Machine
  Learning, {ICML} 2018, Stockholmsm{\"{a}}ssan, Stockholm, Sweden, July 10-15,
  2018}, pages 5389--5398, 2018.

\bibitem{XieXP17}
D.~Xie, J.~Xiong, and S.~Pu.
\newblock All you need is beyond a good init: Exploring better solution for
  training extremely deep convolutional neural networks with orthonormality and
  modulation.
\newblock In {\em 2017 {IEEE} Conference on Computer Vision and Pattern
  Recognition, {CVPR} 2017, Honolulu, HI, USA, July 21-26, 2017}, pages
  5075--5084, 2017.

\bibitem{zavlanos}
M.~M. {Zavlanos} and G.~J. {Pappas}.
\newblock A dynamical systems approach to weighted graph matching.
\newblock In {\em Proceedings of the 45th IEEE Conference on Decision and
  Control}, pages 3492--3497, Dec 2006.

\bibitem{anode_v2}
T.~Zhang, Z.~Yao, A.~Gholami, J.~E. Gonzalez, K.~Keutzer, M.~W. Mahoney, and
  G.~Biros.
\newblock {ANODEV2:} {A} coupled neural {ODE} framework.
\newblock In {\em Advances in Neural Information Processing Systems 32: Annual
  Conference on Neural Information Processing Systems 2019, NeurIPS 2019, 8-14
  December 2019, Vancouver, BC, Canada}, pages 5152--5162, 2019.

\end{thebibliography}
\newpage
\onecolumn
\section*{APPENDIX: An Ode to an ODE}

\section{Proof of Lemma \ref{lemma:grad_vanish_explosion_lemma}}
\begin{proof}
Consider the following Euler-based discretization of the main (non-matrix) flow of the ODEtoODE:
\begin{equation}
\mathbf{x}_{\frac{i+1}{N}} = \mathbf{x}_{\frac{i}{N}} + \frac{1}{N} f \left(\mathbf{W}^{\theta} \left(\frac{i}{N} \right)\mathbf{x}_{\frac{i}{N}} \right),    
\end{equation}

where $i=0,1,.,,,N-1$ and $N \in \mathbb{N}_{+}$ is fixed (defines granularity of discretization). 
Denote: $\mathbf{a}^{N}_{i}=\mathbf{x}_{\frac{i}{N}}$, $\mathbf{c}^{N}_{i}=\mathbf{b}_{\frac{i}{N}}$, $\mathbf{V}^{N}_{i}=\mathbf{W}^{\theta}(\frac{i}{N})$ and
$\mathbf{z}^{N}_{i+1} = \mathbf{V}^{N,\theta}_{i+1}\mathbf{a}^{N}_{i}+\mathbf{c}^{N}_{i+1}$.
We obtain the following discrete dynamical system:
\begin{equation}
\label{recurs}
\mathbf{a}^{N}_{i+1} = \mathbf{a}^{N}_{i} + \frac{1}{N}f \left(\mathbf{z}^{N}_{i+1} \right) ,   
\end{equation}
for $i=0,1,...,N-1$.

Let $\mathcal{L} = \mathcal{L}(\mathbf{x}_{1})=\mathcal{L}(\mathbf{a}^{N}_{N})$ by the loss function.
Our first goal is to compute $\frac{\partial \mathcal{L}}{\partial \mathbf{a}^{N}_{i}}$ for $i=0,...,N-1$.

Using Equation \ref{recurs}, we get:
\begin{align}
\begin{split}
\frac{\partial \mathcal{L}}{\partial \mathbf{a}^{N}_{i}} = 
\frac{\partial \mathbf{a}^{N}_{i+1}}{\partial \mathbf{a}^{N}_{i}}
\frac{\partial \mathcal{L}}{\partial \mathbf{a}^{N}_{i+1}}
=
(\mathbf{I}_{d}+\frac{1}{N}\mathrm{diag}(f^{\prime}(\mathbf{z}^{N}_{i+1}))\frac{\partial \mathbf{z}^{N}_{i+1}}{\partial \mathbf{a}^{N}_{i}}) \frac{\partial \mathcal{L}}{\partial \mathbf{a}^{N}_{i+1}} = \\
(\mathbf{I}_{d}+\frac{1}{N}\mathrm{diag}(f^{\prime}(\mathbf{z}^{N}_{i+1})) \mathbf{V}^{N,\theta}_{i+1}) \frac{\partial \mathcal{L}}{\partial \mathbf{a}^{N}_{i+1}}
\end{split}
\end{align}

Therefore we conclude that:
\begin{equation}
\frac{\partial \mathcal{L}}{\partial \mathbf{a}^{N}_{i}} = \biggl[
\prod_{r=i+1}^{N} 
(\mathbf{I}_{d}+\frac{1}{N}\mathrm{diag}(f^{\prime}(\mathbf{z}^{N}_{r})) \mathbf{V}^{N,\theta}_{r}) \biggr]
\frac{\partial \mathcal{L}}{\partial \mathbf{a}^{N}_{N}}
\end{equation}

Note that for function $f$ such that $|f^{\prime}(x)|=1$ in its differentiable points we have: 
$\mathrm{diag}(f^{\prime}(\mathbf{z}^{N}_{r}))$ is a diagonal matrix with nonzero entries taken from $\{-1, +1\}$, in particular $D \in \mathcal{O}(d)$, where $\mathcal{O}(d)$ stands for the \textit{orthogonal group}.

Define: $\mathbf{G}^{N}_{r} = \mathrm{diag}(f^{\prime}(\mathbf{z}^{N}_{l}))\mathbf{V}^{N,\theta}_{r}$.
Thus we have:
\begin{equation}
\frac{\partial \mathcal{L}}{\partial \mathbf{a}^{N}_{i}} = \biggl[
\prod_{r=i+1}^{N} 
(\mathbf{I}_{d}+\frac{1}{N}\mathbf{G}^{N}_{r}) \biggr]
\frac{\partial \mathcal{L}}{\partial \mathbf{a}^{N}_{N}}
\end{equation}

Note that the following is true:
\begin{align}
\begin{split}
\prod_{r=i+1}^{N} 
(\mathbf{I}_{d}+\frac{1}{N}\mathbf{G}^{N}_{r}) = 
\sum_{\{r_{1},...,r_{k}\} \subseteq \{i+1,...,N\}} \frac{1}{N^{k}}  \mathbf{G}^{N}_{r_{1}} \cdot ... \cdot \mathbf{G}^{N}_{r_{k}}
\end{split}
\end{align}

Therefore we can conclude that:
\begin{align}
\begin{split}
\|\prod_{r=i+1}^{N} 
(\mathbf{I}_{d}+\frac{1}{N}\mathbf{G}^{N}_{r})\|_{2} \leq
\sum_{\{r_{1},...,r_{k}\} \subseteq \{i+1,...,N\}} \frac{1}{N^{k}}  \|\mathbf{G}^{N}_{r_{1}} \cdot ... \cdot \mathbf{G}^{N}_{r_{k}}\|_{2} \\  = \sum_{k=0}^{N-i} \frac{1}{N^{k}}{N-i \choose k} = \sum_{k=0}^{N-i}\frac{(N-i-k+1)\cdot...\cdot (N-i)}{N^{k}k!} \leq \sum_{k=0}^{N-i} \frac{1}{k!} \leq e,
\end{split}    
\end{align}
where we used the fact that $\mathcal{O}(d)$ is a group under matrix-multiplication and $\|\mathbf{G}\|_{2} = 1$ for every $\mathbf{G} \in \mathcal{O}(d)$.
That proves inequality: $\|\frac{\partial \mathcal{L}}{\partial \mathbf{a}^{N}_{i}}\|_{2} \leq e   
\|\frac{\partial \mathcal{L}}{\partial \mathbf{a}^{N}_{N}}\|_{2}$. 

To prove the other inequality: $\|\frac{\partial \mathcal{L}}{\partial \mathbf{a}^{N}_{i}}\|_{2} \geq (\frac{1}{e}-\epsilon)   
\|\frac{\partial \mathcal{L}}{\partial \mathbf{a}^{N}_{N}}\|_{2}$, for large enough $N \geq N(\epsilon)$, it suffices to observe that if $\mathbf{G} \in \mathcal{O}(d)$, then for any $\mathbf{v} \in \mathbb{R}^{d}$ we have (by triangle inequality):
\begin{equation}
\|(\mathbf{I}_{d} + \frac{1}{N}\mathbf{G})\mathbf{v}\|_{2} \geq    
\|\mathbf{v}\|_{2} - \|\frac{1}{N}\mathbf{Gv}\|_{2}=(1-\frac{1}{N})\|\mathbf{v}\|_{2}.
\end{equation}

Lemma \ref{lemma:grad_vanish_explosion_lemma} then follows immediately from the fact that sequence: $\{a_{N}:N=1,2,3,...\}$ defined as: $a_{N}=(1-\frac{1}{N})^{N}$ has limit $e^{-1}$
and by taking $N \rightarrow \infty$.
\end{proof}

\section{Proof of Theorem \ref{th:conv}}

We first prove a number of useful Lemmas. In our derivations we frequently employ an inequality stating that for $\alpha > 0$ $(1 + \frac{\alpha}{N})^N = \exp ( N \log (1 + \frac{\alpha}{N}) ) \leq \exp ( N \cdot \frac{\alpha}{N}) = e^\alpha$ which follows from $\exp(\cdot)$'s monotonicity and $\log (\cdot)$'s concavity and that $\log (1) = 0, \log ' (1) = 1$.
\begin{lem} \label{lemma:locbound}
If Assumption \ref{as:lpssigma} is satisfied, then for any $\theta' = \{ \mathbf{\Omega}_1', \mathbf{\Omega}_2', \mathbf{b}', \mathbf{N}', \mathbf{Q}', \mathbf{W}_0' \} \in \mathbb{D}$ and $\theta'' = \{ \mathbf{\Omega}_1'', \mathbf{\Omega}_2'', \mathbf{b}'', \mathbf{N}'', \mathbf{Q}'', \mathbf{W}_0'' \} \in \mathbb{D}$ such that $\| \mathbf{N}' \|_2, \| \mathbf{Q}' \|_2, \| \mathbf{N}'' \|_2, \| \mathbf{Q}'' \|_2 \leq D, \| \mathbf{b}' \|_2, \| \mathbf{b}'' \|_2 \leq D_b$ for some $D, D_b > 0$ it holds that
\begin{gather}
    \forall \mathbf{s}', \mathbf{s}'' \in \mathbb{R}^d : \| g_{\theta'} (\mathbf{s}') - g_{\theta''} (\mathbf{s}'') \|_2 \leq e \| \mathbf{s}' - \mathbf{s}'' \|_2 \nonumber \\
    + \biggl( e \| \mathbf{s}'' \|_2 + (e - 1) D_b \biggr) \biggl( 1 + (e - 1) ( (1 + \frac{1}{D}) e^{4 D^2} - \frac{1}{D} ) \biggr) \| \theta' - \theta'' \|_2, \label{eq:slips} \\
    \| g_{\theta''} (\mathbf{s}'') \|_2 \leq e \| \mathbf{s}'' \|_2 + (e - 1) D_b . \label{eq:sbound}
\end{gather}
\end{lem}

\begin{proof}
Indeed,
\begin{align}
    \| g_{\theta'} (\mathbf{s}') - g_{\theta''} (\mathbf{s}'') \|_2 &= \| g_{\theta'} (\mathbf{s}') - g_{\theta'} (\mathbf{s}'') + g_{\theta'} (\mathbf{s}'') - g_{\theta''} (\mathbf{s}'') \|_2 \nonumber \\
    &\leq \| g_{\theta'} (\mathbf{s}') - g_{\theta'} (\mathbf{s}'') \|_2 + \| g_{\theta'} (\mathbf{s}'') - g_{\theta''} (\mathbf{s}'') \|_2 . \label{eq:decouple}
\end{align}
Let $\mathbf{x}_1', \dots, \mathbf{x}_N'$ and $\mathbf{x}_1'', \dots, \mathbf{x}_N''$ be rollouts (\ref{eq:roll1}-\ref{eq:roll2}) corresponding to computation of $g_{\theta'} (s')$ and $g_{\theta'} (s'')$ respectively. For any Stiefel matrix $\mathbf{\Omega} \in \mathcal{ST} (d_1, d_2)$ (including square orthogonal matrices) it holds that $\| \mathbf{\Omega} \|_2 = 1$. We use it to deduce:
\begin{align}
    \| g_{\theta'} (\mathbf{s}') &- g_{\theta'} (\mathbf{s}'') \|_2 = \| \mathbf{\Omega}_2' \mathbf{x}_N' - \mathbf{\Omega}_2' \mathbf{x}_N'' \|_2 \leq \| \mathbf{\Omega}_2' \|_2 \| \mathbf{x}_N' - \mathbf{x}_N'' \|_2 = \| \mathbf{x}_N' - \mathbf{x}_N'' \|_2 \nonumber \\
    &= \| \mathbf{x}_{N - 1}' - \mathbf{x}_{N - 1}'' + \frac{1}{N} \biggl( f (\mathbf{W}_N \mathbf{x}_{N - 1}' + \mathbf{b}) - f (\mathbf{W}_N \mathbf{x}_{N - 1}'' + \mathbf{b} ) \biggr) \|_2 \nonumber \\
    &\leq \| \mathbf{x}_{N - 1}' - \mathbf{x}_{N - 1}'' \|_2 + \frac{1}{N} \| f (\mathbf{W}_N \mathbf{x}_{N - 1}' + \mathbf{b}) - f (\mathbf{W}_N \mathbf{x}_{N - 1}'' + \mathbf{b} ) \|_2 \nonumber \\
    &\leq \| \mathbf{x}_{N - 1}' - \mathbf{x}_{N - 1}'' \|_2 + \frac{1}{N} \| \mathbf{W}_N \mathbf{x}_{N - 1}' - \mathbf{W}_N \mathbf{x}_{N - 1}'' \|_2 \nonumber \\
    &= \| \mathbf{x}_{N - 1}' - \mathbf{x}_{N - 1}'' \|_2 + \frac{1}{N} \| \mathbf{x}_{N - 1}' - \mathbf{x}_{N - 1}'' \|_2 \nonumber \\
    &= (1 + \frac{1}{N}) \| \mathbf{x}_{N - 1}' - \mathbf{x}_{N - 1}'' \|_2 \leq \dots \leq (1 + \frac{1}{N})^N \| \mathbf{x}_0' - \mathbf{x}_0'' \|_2 \nonumber \\
    &\leq e \| \mathbf{x}_0' - \mathbf{x}_0'' \|_2^2 = e \| \mathbf{\Omega}_1' \mathbf{s}' - \mathbf{\Omega}_1' \mathbf{s}'' \|_2^2 \leq e \| \mathbf{\Omega}_1' \|_2 \| \mathbf{s}' - \mathbf{s}'' \|_2^2 = e \| \mathbf{s}' - \mathbf{s}'' \|_2^2 . \label{eq:bnd1}
\end{align}
Let $\mathbf{x}_1', \mathbf{W}_1', \dots, \mathbf{x}_N', \mathbf{W}_N'$ and $\mathbf{x}_1'', \mathbf{W}_1'', \dots, \mathbf{x}_N'', \mathbf{W}_N''$ be rollouts (\ref{eq:roll1}-\ref{eq:roll2}) corresponding to computation of $g_{\theta'} (s'')$ and $g_{\theta''} (s'')$ respectively. We fix $i \in \{ 1, \dots, N \}$ and deduce, using Assumption \ref{as:lpssigma} in particular, that
\begin{align}
    \| \mathbf{x}_i'' \|_2 &\leq \| \mathbf{x}_{i - 1}'' \|_2 + \frac{1}{N} \| f (\mathbf{W}_i'' \mathbf{x}_{i - 1}'' + \mathbf{b}'') \|_2 \leq \| \mathbf{x}_{i - 1}'' \|_2 + \frac{1}{N} \| \mathbf{W}_i'' \mathbf{x}_{i - 1}'' + \mathbf{b}'' \|_2 \nonumber \\
    &\leq \| \mathbf{x}_{i - 1}'' \|_2 + \frac{1}{N} \| \mathbf{W}_i'' \mathbf{x}_{i - 1}'' \|_2 + \frac{1}{N} \| \mathbf{b}'' \|_2 = (1 + \frac{1}{N}) \| \mathbf{x}_{i - 1}'' \|_2 + \frac{1}{N} \| \mathbf{b}'' \|_2 \nonumber \\
    &\leq (1 + \frac{1}{N})^i \| \mathbf{x}_0'' \|_2 + \frac{1}{N} \| \mathbf{b}'' \|_2 \sum_{j = 0}^{i - 1} (1 + \frac{1}{N})^j = (1 + \frac{1}{N})^i \| \mathbf{\Omega}_1'' \mathbf{s}'' \|_2 \nonumber \\
    &+ ((1 + \frac{1}{N})^i - 1) \| \mathbf{b}'' \|_2 = (1 + \frac{1}{N})^i \| \mathbf{\Omega}_1'' \|_2 \| \mathbf{s}'' \|_2 + ((1 + \frac{1}{N})^i - 1) D_b \nonumber \\
    &\leq (1 + \frac{1}{N})^N \| \mathbf{s}'' \|_2 + ((1 + \frac{1}{N})^N - 1) D_b \nonumber \\
    &\leq e \| \mathbf{s}'' \|_2 + (e - 1) D_b . \label{eq:xbound}
\end{align}
As a particular case of (\ref{eq:xbound}) when $i = N$ and $\| g_{\theta''} (\mathbf{s}'') \|_2 = \| \mathbf{\Omega}_2'' \mathbf{x}_N'' \|_2 \leq \| \mathbf{x}_N'' \|_2$ we obtain (\ref{eq:sbound}). Set
\begin{gather*}
    \mathbf{A}' = \frac{1}{N} ( \mathbf{W}_{i - 1}'^\top \mathbf{Q}' \mathbf{W}_{i - 1}' \mathbf{N}' - \mathbf{N}' \mathbf{W}_{i - 1}'^\top \mathbf{Q}' \mathbf{W}_{i - 1}' ), \\
    \mathbf{A}'' = \frac{1}{N} ( \mathbf{W}_{i - 1}''^\top \mathbf{Q}'' \mathbf{W}_{i - 1}'' \mathbf{N}'' - \mathbf{N}'' \mathbf{W}_{i - 1}''^\top \mathbf{Q}'' \mathbf{W}_{i - 1}'' ).
\end{gather*}
Since $\mathbf{A}', \mathbf{A}'' \in \mathrm{Skew} (d)$ and $\mathrm{Skew} (d)$ is a vector space, we conclude that $\exp (\alpha \mathbf{A}' + \alpha t (\mathbf{A}'' - \mathbf{A}')) \in \mathcal{O} (d)$ for any $t, \alpha \in \mathbb{R}$ where we use that $\exp(\cdot)$ maps $\mathrm{Skew} (d)$ into $\mathcal{O}(d)$. We also use a rule \cite{wilcox} which states that for $\mathbf{X}: \mathbb{R} \to \mathbb{R}^{d \times d}$
\begin{equation*}
    \frac{d}{d t} \exp (\mathbf{X} (t)) = \int_0^1 \exp (\alpha  \mathbf{X} (t) ) \frac{d \mathbf{X} (t)}{d t} \exp ((1 - \alpha) \mathbf{X} (t)) d \alpha
\end{equation*}

to deduce that
\begin{align*}
    \| \exp (\mathbf{A}') - \exp (\mathbf{A''}) \|_2^2 &= \| \int_{t = 0}^1 \frac{d}{d t} \exp (\mathbf{A}'' + t (\mathbf{A}' - \mathbf{A}'')) d t \|_2^2 \\
    &= \| \int_0^1 \int_0^1 \exp \biggl( \alpha \mathbf{A}'' + \alpha t (\mathbf{A}' - \mathbf{A}'') \biggr) (\mathbf{A}' - \mathbf{A}'') \exp \biggl( (1 - \alpha) \mathbf{A}'' \\
    &+ (1 - \alpha) t (\mathbf{A}' - \mathbf{A}'') \biggr) d \alpha d t \|_2^2 \\
    &\leq \int_0^1 \int_0^1 \| \exp \biggl( \alpha \mathbf{A}'' + \alpha t (\mathbf{A}' - \mathbf{A}'') \biggr) (\mathbf{A}' - \mathbf{A}'') \exp \biggl( (1 - \alpha) \mathbf{A}'' \\
    &+ (1 - \alpha) t (\mathbf{A}' - \mathbf{A}'') \biggr) \|_2^2 d \alpha d t \\
    &= \int_0^1 \int_0^1 \| \mathbf{A}' - \mathbf{A}'' \|_2^2 d \alpha d t = \| \mathbf{A}' - \mathbf{A}'' \|_2^2 .
\end{align*}
Consequently, we derive that
\begin{align*}
    &\| \mathbf{W}_i' - \mathbf{W}_i'' \|_2 = \| \mathbf{W}_{i - 1}' \exp (\mathbf{A}') - \mathbf{W}_{i - 1}'' \exp (\mathbf{A}'') \|_2 \\
    &= \| \mathbf{W}_{i - 1}' \exp (\mathbf{A}') - \mathbf{W}_{i - 1}'' \exp (\mathbf{A}') + \mathbf{W}_{i - 1}'' \exp (\mathbf{A}') - \mathbf{W}_{i - 1}'' \exp (\mathbf{A}'') \|_2 \\
    &\leq \| \mathbf{W}_{i - 1}' \exp (\mathbf{A}') - \mathbf{W}_{i - 1}'' \exp (\mathbf{A}') \|_2 + \| \mathbf{W}_{i - 1}'' \exp (\mathbf{A}') - \mathbf{W}_{i - 1}'' \exp (\mathbf{A}'') \|_2 \\
    &= \| \mathbf{W}_{i - 1}' - \mathbf{W}_{i - 1}'' \|_2 + \| \exp (\mathbf{A}') - \exp (\mathbf{A}'') \|_2 \leq \| \mathbf{W}_{i - 1}' - \mathbf{W}_{i - 1}'' \|_2 + \| \mathbf{A}' - \mathbf{A}'' \|_2 \\
    &= \| \mathbf{W}_{i - 1}' - \mathbf{W}_{i - 1}'' \|_2 + \frac{1}{N} \| (\mathbf{W}_{i - 1}'^\top \mathbf{Q}' \mathbf{W}_{i - 1}' \mathbf{N}' - \mathbf{W}_{i - 1}''^\top \mathbf{Q}'' \mathbf{W}_{i - 1}'' \mathbf{N}'') - (\mathbf{N}' \mathbf{W}_{i - 1}'^\top \mathbf{Q}' \mathbf{W}_{i - 1}' \\
    &- \mathbf{N}'' \mathbf{W}_{i - 1}''^\top \mathbf{Q}'' \mathbf{W}_{i - 1}'') \|_2 \\
    &\leq \| \mathbf{W}_{i - 1}' - \mathbf{W}_{i - 1}'' \|_2 + \frac{1}{N} \| \mathbf{W}_{i - 1}'^\top \mathbf{Q}' \mathbf{W}_{i - 1}' \mathbf{N}' - \mathbf{W}_{i - 1}''^\top \mathbf{Q}'' \mathbf{W}_{i - 1}'' \mathbf{N}'' \|_2 + \frac{1}{N} \| \mathbf{N}' \mathbf{W}_{i - 1}'^\top \mathbf{Q}' \mathbf{W}_{i - 1}' \\
    &- \mathbf{N}'' \mathbf{W}_{i - 1}''^\top \mathbf{Q}'' \mathbf{W}_{i - 1}'' \|_2 \\
    &= \| \mathbf{W}_{i - 1}' - \mathbf{W}_{i - 1}'' \|_2 + \frac{1}{N} \| \mathbf{W}_{i - 1}'^\top \mathbf{Q}' \mathbf{W}_{i - 1}' \mathbf{N}' - \mathbf{W}_{i - 1}'^\top \mathbf{Q}' \mathbf{W}_{i - 1}'' \mathbf{N}'' + \mathbf{W}_{i - 1}'^\top \mathbf{Q}' \mathbf{W}_{i - 1}'' \mathbf{N}'' \\
    &- \mathbf{W}_{i - 1}''^\top \mathbf{Q}'' \mathbf{W}_{i - 1}'' \mathbf{N}'' \|_2 + \frac{1}{N} \| \mathbf{N}' \mathbf{W}_{i - 1}'^\top \mathbf{Q}' \mathbf{W}_{i - 1}' - \mathbf{N}' \mathbf{W}_{i - 1}'^\top \mathbf{Q}'' \mathbf{W}_{i - 1}'' + \mathbf{N}' \mathbf{W}_{i - 1}'^\top \mathbf{Q}'' \mathbf{W}_{i - 1}'' \\
    &- \mathbf{N}'' \mathbf{W}_{i - 1}''^\top \mathbf{Q}'' \mathbf{W}_{i - 1}'' \|_2 \\
    &\leq \| \mathbf{W}_{i - 1}' - \mathbf{W}_{i - 1}'' \|_2 + \frac{1}{N} \| \mathbf{W}_{i - 1}'^\top \mathbf{Q}' \mathbf{W}_{i - 1}' \mathbf{N}' - \mathbf{W}_{i - 1}'^\top \mathbf{Q}' \mathbf{W}_{i - 1}'' \mathbf{N}'' \|_2 + \frac{1}{N} \| \mathbf{W}_{i - 1}'^\top \mathbf{Q}' \mathbf{W}_{i - 1}'' \mathbf{N}'' \\
    &- \mathbf{W}_{i - 1}''^\top \mathbf{Q}'' \mathbf{W}_{i - 1}'' \mathbf{N}'' \|_2 + \frac{1}{N} \| \mathbf{N}' \mathbf{W}_{i - 1}'^\top \mathbf{Q}' \mathbf{W}_{i - 1}' - \mathbf{N}' \mathbf{W}_{i - 1}'^\top \mathbf{Q}'' \mathbf{W}_{i - 1}'' \|_2 + \frac{1}{N} \| \mathbf{N}' \mathbf{W}_{i - 1}'^\top \mathbf{Q}'' \mathbf{W}_{i - 1}'' \\
    &- \mathbf{N}'' \mathbf{W}_{i - 1}''^\top \mathbf{Q}'' \mathbf{W}_{i - 1}'' \|_2 \\
    &\leq \| \mathbf{W}_{i - 1}' - \mathbf{W}_{i - 1}'' \|_2 + \frac{1}{N} \| \mathbf{Q}' \|_2 \| \mathbf{W}_{i - 1}' \mathbf{N}' - \mathbf{W}_{i - 1}'' \mathbf{N}'' \|_2 + \frac{1}{N} \| \mathbf{N}'' \|_2 \| \mathbf{W}_{i - 1}'^\top \mathbf{Q}' - \mathbf{W}_{i - 1}''^\top \mathbf{Q}'' \|_2 \\
    &+ \frac{1}{N} \| \mathbf{N}' \|_2 \| \mathbf{Q}' \mathbf{W}_{i - 1}' - \mathbf{Q}'' \mathbf{W}_{i - 1}'' \|_2 + \frac{1}{N} \| \mathbf{Q}'' \|_2 \| \mathbf{N}' \mathbf{W}_{i - 1}'^\top - \mathbf{N}'' \mathbf{W}_{i - 1}''^\top \|_2 \\
    &\leq \| \mathbf{W}_{i - 1}' - \mathbf{W}_{i - 1}'' \|_2 + \frac{D}{N} \| \mathbf{W}_{i - 1}' \mathbf{N}' - \mathbf{W}_{i - 1}'' \mathbf{N}'' \|_2 + \frac{D}{N} \| \mathbf{W}_{i - 1}'^\top \mathbf{Q}' - \mathbf{W}_{i - 1}''^\top \mathbf{Q}'' \|_2 \\
    &+ \frac{D}{N} \| \mathbf{Q}' \mathbf{W}_{i - 1}' - \mathbf{Q}'' \mathbf{W}_{i - 1}'' \|_2 + \frac{D}{N} \| \mathbf{N}' \mathbf{W}_{i - 1}'^\top - \mathbf{N}'' \mathbf{W}_{i - 1}''^\top \|_2 \\
    &= \| \mathbf{W}_{i - 1}' - \mathbf{W}_{i - 1}'' \|_2 + \frac{D}{N} \| \mathbf{W}_{i - 1}' \mathbf{N}' - \mathbf{W}_{i - 1}' \mathbf{N}'' + \mathbf{W}_{i - 1}' \mathbf{N}'' - \mathbf{W}_{i - 1}'' \mathbf{N}'' \|_2 + \frac{D}{N} \| \mathbf{W}_{i - 1}'^\top \mathbf{Q}' \\
    &- \mathbf{W}_{i - 1}'^\top \mathbf{Q}'' + \mathbf{W}_{i - 1}'^\top \mathbf{Q}'' - \mathbf{W}_{i - 1}''^\top \mathbf{Q}'' \|_2 + \frac{D}{N} \| \mathbf{Q}' \mathbf{W}_{i - 1}' - \mathbf{Q}' \mathbf{W}_{i - 1}'' + \mathbf{Q}' \mathbf{W}_{i - 1}'' - \mathbf{Q}'' \mathbf{W}_{i - 1}'' \|_2 \\
    &+ \frac{D}{N} \| \mathbf{N}' \mathbf{W}_{i - 1}'^\top - \mathbf{N}' \mathbf{W}_{i - 1}''^\top + \mathbf{N}' \mathbf{W}_{i - 1}''^\top - \mathbf{N}'' \mathbf{W}_{i - 1}''^\top \|_2 \\
    &\leq \| \mathbf{W}_{i - 1}' - \mathbf{W}_{i - 1}'' \|_2 + \frac{D}{N} \| \mathbf{W}_{i - 1}' \mathbf{N}' - \mathbf{W}_{i - 1}' \mathbf{N}'' \|_2 + \frac{D}{N} \| \mathbf{W}_{i - 1}' \mathbf{N}'' - \mathbf{W}_{i - 1}'' \mathbf{N}'' \|_2 \\
    &+ \frac{D}{N} \| \mathbf{W}_{i - 1}'^\top \mathbf{Q}' - \mathbf{W}_{i - 1}'^\top \mathbf{Q}'' \|_2 + \frac{D}{N} \| \mathbf{W}_{i - 1}'^\top \mathbf{Q}'' - \mathbf{W}_{i - 1}''^\top \mathbf{Q}'' \|_2 + \frac{D}{N} \| \mathbf{Q}' \mathbf{W}_{i - 1}' - \mathbf{Q}' \mathbf{W}_{i - 1}'' \|_2 \\
    &+ \frac{D}{N} \| \mathbf{Q}' \mathbf{W}_{i - 1}'' - \mathbf{Q}'' \mathbf{W}_{i - 1}'' \|_2 + \frac{D}{N} \| \mathbf{N}' \mathbf{W}_{i - 1}'^\top - \mathbf{N}' \mathbf{W}_{i - 1}''^\top \|_2 + \frac{D}{N} \| \mathbf{N}' \mathbf{W}_{i - 1}''^\top - \mathbf{N}'' \mathbf{W}_{i - 1}''^\top \|_2 \\
    &\leq \| \mathbf{W}_{i - 1}' - \mathbf{W}_{i - 1}'' \|_2 + \frac{D}{N} \| \mathbf{N}' - \mathbf{N}'' \|_2 + \frac{D}{N} \| \mathbf{N}'' \|_2 \| \mathbf{W}_{i - 1}' - \mathbf{W}_{i - 1}'' \|_2 + \frac{D}{N} \| \mathbf{Q}' - \mathbf{Q}'' \|_2 \\
    &+ \frac{D}{N} \| \mathbf{Q}'' \|_2 \| \mathbf{W}_{i - 1}'^\top  - \mathbf{W}_{i - 1}''^\top \|_2 + \frac{D}{N} \| \mathbf{Q}' \|_2 \| \mathbf{W}_{i - 1}' - \mathbf{W}_{i - 1}'' \|_2 + \frac{D}{N} \| \mathbf{Q}' - \mathbf{Q}'' \|_2 \\
    &+ \frac{D}{N} \| \mathbf{N}' \|_2 \| \mathbf{W}_{i - 1}'^\top - \mathbf{W}_{i - 1}''^\top \|_2 + \frac{D}{N} \| \mathbf{N}' - \mathbf{N}'' \|_2 \\
    &\leq (1 + 4 \frac{D^2}{N}) \| \mathbf{W}_{i - 1}' - \mathbf{W}_{i - 1}'' \|_2 + 2 \frac{D}{N} \| \mathbf{N}' - \mathbf{N}'' \|_2 + 2 \frac{D}{N} \| \mathbf{Q}' - \mathbf{Q}'' \|_2 \\
    &\leq (1 + 4 \frac{D^2}{N}) \| \mathbf{W}_{i - 1}' - \mathbf{W}_{i - 1}'' \|_2 + 4 \frac{D}{N} \| \theta' - \theta'' \|_2 \leq \dots \\
    &\leq (1 + 4 \frac{D^2}{N})^i \| \mathbf{W}_0' - \mathbf{W}_0'' \|_2 + 4 \frac{D}{N} \sum_{j = 0}^{i - 1} (1 + 4 \frac{D^2}{N})^j \| \theta' - \theta'' \|_2 \\
    &= (1 + 4 \frac{D^2}{N})^i \| \mathbf{W}_0' - \mathbf{W}_0'' \|_2 + \frac{1}{D} ((1 + 4 \frac{D^2}{N})^i - 1) \| \theta' - \theta'' \|_2 \\
    &\leq (1 + 4 \frac{D^2}{N})^i \| \theta' - \theta'' \|_2 + \frac{1}{D} ((1 + 4 \frac{D^2}{N})^i - 1) \| \theta' - \theta'' \|_2 \\
    &= \biggl( (1 + \frac{1}{D}) (1 + 4 \frac{D^2}{N})^i - \frac{1}{D} \biggr) \| \theta' - \theta'' \|_2 \leq \biggl( (1 + \frac{1}{D}) (1 + 4 \frac{D^2}{N})^N - \frac{1}{D} \biggr) \| \theta' - \theta'' \|_2 \\
    &\leq \biggl( (1 + \frac{1}{D}) e^{4 D^2} - \frac{1}{D} \biggr) \| \theta' - \theta'' \|_2
\end{align*}
We use (\ref{eq:xbound}) and derive that
\begin{align*}
    &\| \mathbf{x}_i' - \mathbf{x}_i'' \|_2 = \| \mathbf{x}_{i - 1}' - \mathbf{x}_{i - 1}'' + \frac{1}{N} (f(\mathbf{W}_i' \mathbf{x}_{i - 1}' + \mathbf{b}) - f(\mathbf{W}_i'' \mathbf{x}_{i - 1}'' + \mathbf{b})) \|_2 \\
    &\leq \| \mathbf{x}_{i - 1}' - \mathbf{x}_{i - 1}'' \|_2 + \frac{1}{N} \| f(\mathbf{W}_i' \mathbf{x}_{i - 1}' + \mathbf{b}) - f(\mathbf{W}_i'' \mathbf{x}_{i - 1}'' + \mathbf{b}) \|_2 \\
    &\leq \| \mathbf{x}_{i - 1}' - \mathbf{x}_{i - 1}'' \|_2 + \frac{1}{N} \| \mathbf{W}_i' \mathbf{x}_{i - 1}' - \mathbf{W}_i'' \mathbf{x}_{i - 1}'' \|_2 \\
    &\leq \| \mathbf{x}_{i - 1}' - \mathbf{x}_{i - 1}'' \|_2 + \frac{1}{N} \| \mathbf{W}_i' \mathbf{x}_{i - 1}' - \mathbf{W}_i' \mathbf{x}_{i - 1}'' + \mathbf{W}_i' \mathbf{x}_{i - 1}'' - \mathbf{W}_i'' \mathbf{x}_{i - 1}'' \|_2 \\
    &\leq \| \mathbf{x}_{i - 1}' - \mathbf{x}_{i - 1}'' \|_2 + \frac{1}{N} \| \mathbf{W}_i' \mathbf{x}_{i - 1}' - \mathbf{W}_i' \mathbf{x}_{i - 1}'' \|_2 + \frac{1}{N} \| \mathbf{W}_i' \mathbf{x}_{i - 1}'' - \mathbf{W}_i'' \mathbf{x}_{i - 1}'' \|_2 \\
    &\leq (1 + \frac{1}{N}) \| \mathbf{x}_{i - 1}' - \mathbf{x}_{i - 1}'' \|_2 + \frac{1}{N} \| \mathbf{x}_{i - 1}'' \|_2 \| \mathbf{W}_i' - \mathbf{W}_i'' \|_2 \\
    &\leq (1 + \frac{1}{N}) \| \mathbf{x}_{i - 1}' - \mathbf{x}_{i - 1}'' \|_2 + \frac{1}{N} \biggl( e \| \mathbf{s}'' \|_2 + (e - 1) D_b \biggr) \| \mathbf{W}_i' - \mathbf{W}_i'' \|_2 \\
    &\leq (1 + \frac{1}{N}) \| \mathbf{x}_{i - 1}' - \mathbf{x}_{i - 1}'' \|_2 \\
    &+ \frac{1}{N} \biggl( e \| \mathbf{s}'' \|_2 + (e - 1) D_b \biggr) \biggl( (1 + \frac{1}{D}) e^{4 D^2} - \frac{1}{D} \biggr) \| \theta' - \theta'' \|_2
\end{align*}
By aggregating the last inequality for $i \in \{ 1, \dots, N \}$ we conclude that
\begin{align*}
    &\| g_{\theta'} (\mathbf{s}'') - g_{\theta''} (\mathbf{s}'') \|_2 = \| \mathbf{\Omega}_2' \mathbf{x}_N' - \mathbf{\Omega}_2'' \mathbf{x}_N'' \|_2 = \| \mathbf{\Omega}_2' \mathbf{x}_N' - \mathbf{\Omega}_2' \mathbf{x}_N'' + \mathbf{\Omega}_2' \mathbf{x}_N'' - \mathbf{\Omega}_2'' \mathbf{x}_N'' \|_2 \\
    &\leq \| \mathbf{\Omega}_2' \|_2 \| \mathbf{x}_N' - \mathbf{x}_N'' \|_2 + \| \mathbf{x}_N'' \|_2 \| \mathbf{\Omega}_2' - \mathbf{\Omega}_2'' \|_2 \\
    &\leq \| \mathbf{x}_N' - \mathbf{x}_N'' \|_2 + \biggl( e \| \mathbf{s}'' \|_2 + (e - 1) D_b \biggr) \| \theta' - \theta'' \|_2 \\
    &\leq (1 + \frac{1}{N})^N \| \mathbf{x}_0' - \mathbf{x}_0'' \|_2 \\
    &+ \biggl( e \| \mathbf{s}'' \|_2 + (e - 1) D_b \biggr) \biggl( 1 +  \sum_{j = 0}^{i - 1} (1 + \frac{1}{N})^j \cdot \frac{1}{N} ( (1 + \frac{1}{D}) e^{4 D^2} - \frac{1}{D} ) \biggr) \| \theta' - \theta'' \|_2 \\
    &\leq (1 + \frac{1}{N})^N \| \mathbf{s}'' - \mathbf{s}'' \|_2 \\
    &+ \biggl( e \| \mathbf{s}'' \|_2 + (e - 1) D_b \biggr) \biggl( 1 + ((1 + \frac{1}{N})^N - 1) ( (1 + \frac{1}{D}) e^{4 D^2} - \frac{1}{D} ) \biggr) \| \theta' - \theta'' \|_2 \\
    &\leq \biggl( e \| \mathbf{s}'' \|_2 + (e - 1) D_b \biggr) \biggl( 1 + (e - 1) ( (1 + \frac{1}{D}) e^{4 D^2} - \frac{1}{D} ) \biggr) \| \theta' - \theta'' \|_2
\end{align*}
The inequality above together with (\ref{eq:decouple}) and (\ref{eq:bnd1}) concludes the proof of bound (\ref{eq:slips}).
\end{proof}

\begin{lem} \label{lemma:lips}
If Assumptions \ref{as:lpsf}, \ref{as:lpssigma} are satisfied, then for any $\theta' = \{ \mathbf{\Omega}_1', \mathbf{\Omega}_2', \mathbf{b}', \mathbf{N}', \mathbf{Q}', \mathbf{W}_0' \}  \in \mathbb{D}$ and $\theta'' = \{ \mathbf{\Omega}_1'', \mathbf{\Omega}_2'', \mathbf{b}'', \mathbf{N}'', \mathbf{Q}'', \mathbf{W}_0'' \}  \in \mathbb{D}$ such that $\| \mathbf{N}' \|_2, \| \mathbf{Q}' \|_2, \| \mathbf{N}'' \|_2, \| \mathbf{Q}'' \|_2 \leq D, \| \mathbf{b}' \|_2, \| \mathbf{b}'' \|_2 \leq D_b$ for some $D, D_b > 0$ it holds that
\begin{equation*}
    | F(\theta') - F(\theta'') | \leq \mathcal{C} \| \theta' - \theta'' \|_2
\end{equation*}
where
\begin{gather*}
    \mathcal{C} = L_2 K ((1 + e) \gamma(K) L_1 + 1) \biggl( e ( L_1^K (1 + e)^K + 1) \| \mathbf{s}_0 \|_2 + \gamma(K) \biggl( L_1 ( (e - 1) D_b + \| \mathbf{s}_0 \|_2 \\
    + \| \widehat{\mathbf{a}} \|_2 ) + \| \mathrm{env}^{(1)} (\mathbf{s}_0, \widehat{\mathbf{a}}) \|_2 \biggr) + (e - 1) D_b \biggr) \biggl( 1 + (e - 1) ( (1 + \frac{1}{D}) e^{4 D^2} - \frac{1}{D} ) \biggr) \\
    \gamma (k) = \begin{cases}
      k, & \text{if}\ L_1 (1 + e) = 1 \\
      \frac{L_1^k (1 + e)^k - 1}{L_1 (1 + e) - 1}, & \text{otherwise}
    \end{cases}
\end{gather*}
and $\widehat{\mathbf{a}}$ is an arbitrary fixed vector from $\mathbb{R}^m$.
\end{lem}
\begin{proof}
Let $\mathbf{s}_1', l_1' \dots, \mathbf{s}_K', l_K'$ and $\mathbf{s}_0'', l_1'' \dots, \mathbf{s}_K'', l_K''$ be rollouts of (\ref{eq:esrollout}) for $\theta'$ and $\theta''$ respectively starting from $\mathbf{s}_0' = \mathbf{s}_0'' = \mathbf{s}_0$. In the light of Assumption \ref{as:lpsf} and Lemma \ref{lemma:locbound} for any $k \in \{ 1, \dots, K \}$ we have $\mathbf{s}_k'' = \mathrm{env}^{(1)} (\mathbf{s}_{k - 1}'', g_{\theta''} (\mathbf{s}_{k - 1}''))$ and, therefore,
\begin{align*}
    &\| \mathbf{s}_k'' \|_2 = \| \mathrm{env}^{(1)} (\mathbf{s}_{k - 1}'', g_{\theta''} (\mathbf{s}_{k - 1}'')) - \mathrm{env}^{(1)} (\mathbf{s}_0, \widehat{\mathbf{a}}) + \mathrm{env}^{(1)} (\mathbf{s}_0, \widehat{\mathbf{a}}) \|_2 \\
    &\leq \| \mathrm{env}^{(1)} (\mathbf{s}_{k - 1}'', g_{\theta''} (\mathbf{s}_{k - 1}'')) - \mathrm{env}^{(1)} (\mathbf{s}_0, \widehat{\mathbf{a}}) \|_2 + \| \mathrm{env}^{(1)} (\mathbf{s}_0, \widehat{\mathbf{a}}) \|_2 \\
    &\leq L_1 \biggl( \| \mathbf{s}_{k - 1}'' - \mathbf{s}_0 \|_2 + \| g_{\theta''} (\mathbf{s}_{k - 1}'') - \widehat{\mathbf{a}} \|_2 \biggr) + \| \mathrm{env}^{(1)} (\mathbf{s}_0, \widehat{\mathbf{a}}) \|_2 \\
    &\leq L_1 \| \mathbf{s}_{k - 1}'' \|_2 + L_1 \| \mathbf{s}_0 \|_2 + L_1 \| g_{\theta''} (\mathbf{s}_{k - 1}'') \|_2 + L_1 \| \widehat{\mathbf{a}} \|_2 + \| \mathrm{env}^{(1)} (\mathbf{s}_0, \widehat{\mathbf{a}}) \|_2 \\
    &\leq L_1 (1 + e) \| \mathbf{s}_{k - 1}'' \|_2 + L_1 ( (e - 1) D_b + \| \mathbf{s}_0 \|_2 + \| \widehat{\mathbf{a}} \|_2 ) + \| \mathrm{env}^{(1)} (\mathbf{s}_0, \widehat{\mathbf{a}}) \|_2 \leq \dots \\
    &\leq L_1^k (1 + e)^k \| \mathbf{s}_0 \|_2 + \sum_{j = 0}^{k - 1} L_1^j (1 + e)^j \biggl( L_1 ( (e - 1) D_b + \| \mathbf{s}_0 \|_2 + \| \widehat{\mathbf{a}} \|_2 ) + \| \mathrm{env}^{(1)} (\mathbf{s}_0, \widehat{\mathbf{a}}) \|_2 \biggr) \\
    &\leq L_1^k (1 + e)^k \| \mathbf{s}_0 \|_2 + \gamma (k) \biggl( L_1 ( (e - 1) D_b + \| \mathbf{s}_0 \|_2 + \| \widehat{\mathbf{a}} \|_2 ) + \| \mathrm{env}^{(1)} (\mathbf{s}_0, \widehat{\mathbf{a}}) \|_2 \biggr) \\
    &\leq L_1^K (1 + e)^K \| \mathbf{s}_0 \|_2 + \gamma (K) \biggl( L_1 ( (e - 1) D_b + \| \mathbf{s}_0 \|_2 + \| \widehat{\mathbf{a}} \|_2 ) + \| \mathrm{env}^{(1)} (\mathbf{s}_0, \widehat{\mathbf{a}}) \|_2 \biggr) \\
    &= L_1^K (1 + e)^K \| \mathbf{s}_0 \|_2 + \gamma(K) \mathcal{A}
\end{align*}
where we denote
\begin{equation*}
    \mathcal{A} = L_1 ( (e - 1) D_b + \| \mathbf{s}_0 \|_2 + \| \widehat{\mathbf{a}} \|_2 ) + \| \mathrm{env}^{(1)} (\mathbf{s}_0, \widehat{\mathbf{a}}) \|_2 .
\end{equation*}
In addition to that, denote
\begin{equation*}
    \mathcal{B} = 1 + (e - 1) ( (1 + \frac{1}{D}) e^{4 D^2} - \frac{1}{D} ).
\end{equation*}
From the last inequality it follows that
\begin{equation*}
    \| \mathbf{s}_{k - 1}'' \|_2 \leq \max (\| \mathbf{s}_0 \|_2, L_1^K (1 + e)^K \| \mathbf{s}_0 \|_2 + \gamma(K) \mathcal{A}) \leq ( L_1^K (1 + e)^K + 1) \| \mathbf{s}_0 \|_2 + \gamma(K) \mathcal{A}
\end{equation*}
Next, observe that
\begin{align*}
    &\| \mathbf{s}_k' - \mathbf{s}_k'' \|_2 = \| \mathrm{env}^{(1)} (\mathbf{s}_{k - 1}', g_{\theta'} (\mathbf{s}_{k - 1}')) - \mathrm{env}^{(1)} (\mathbf{s}_{k - 1}'', g_{\theta''} (\mathbf{s}_{k - 1}'')) \|_2 \\
    &\leq L_1 \biggl( \| \mathbf{s}_{k - 1}' - \mathbf{s}_{k - 1}'' \|_2 + \| g_{\theta'} (\mathbf{s}_{k - 1}') - g_{\theta''} (\mathbf{s}_{k - 1}'') \|_2 \biggr) \\
    &\leq L_1 (1 + e) \| \mathbf{s}_{k - 1}' - \mathbf{s}_{k - 1}'' \|_2 \\
    &+ L_1 \biggl( e \| \mathbf{s}_{k - 1}'' \|_2 + (e - 1) D_b \biggr) \biggl( 1 + (e - 1) ( (1 + \frac{1}{D}) e^{4 D^2} - \frac{1}{D} ) \biggr) \| \theta' - \theta'' \|_2 \\
    &\leq L_1 (1 + e) \| \mathbf{s}_{k - 1}' - \mathbf{s}_{k - 1}'' \|_2 + L_1 \biggl( e ( ( L_1^K (1 + e)^K + 1) \| \mathbf{s}_0 \|_2 + \gamma(K) \mathcal{A}) \\
    &+ (e - 1) D_b \biggr) \mathcal{B} \| \theta' - \theta'' \|_2 \leq \dots \\
    &\leq L_1^k (1 + e)^k \| \mathbf{s}_0 - \mathbf{s}_0 \|_2 + \sum_{j = 0}^{k - 1} L_1^j (1 + e)^j L_1 \biggl( e ( ( L_1^K (1 + e)^K + 1) \| \mathbf{s}_0 \|_2 + \gamma(K) \mathcal{A}) \\
    &+ (e - 1) D_b \biggr) \mathcal{B} \| \theta' - \theta'' \|_2 \\
    &= \gamma (k) L_1 \biggl( e ( ( L_1^K (1 + e)^K + 1) \| \mathbf{s}_0 \|_2 + \gamma(K) \mathcal{A}) + (e - 1) D_b \biggr) \mathcal{B} \| \theta' - \theta'' \|_2 \\
    &\leq \gamma (K) L_1 \biggl( e ( ( L_1^K (1 + e)^K + 1) \| \mathbf{s}_0 \|_2 + \gamma(K) \mathcal{A}) + (e - 1) D_b \biggr) \mathcal{B} \| \theta' - \theta'' \|_2.
\end{align*}
We conclude by deriving that
\begin{align*}
    &| F(\theta') - F(\theta'') | \leq \sum_{k = 1}^K | l_k' - l_k'' | \leq \sum_{k = 1}^K | \mathrm{env}^{(2)} (\mathbf{s}_{k - 1}', g_{\theta'} (\mathbf{s}_{k - 1}')) - \mathrm{env}^{(2)} (\mathbf{s}_{k - 1}'', g_{\theta''} (\mathbf{s}_{k - 1}'')) | \\
    &\leq \sum_{k = 1}^K L_2 \biggl( \| \mathbf{s}_{k - 1}' - \mathbf{s}_{k - 1}'' \|_2 + \| g_{\theta'} (\mathbf{s}_{k - 1}') - g_{\theta''} (\mathbf{s}_{k - 1}'') \|_2 \biggr) \\
    &\leq L_2 \sum_{k = 1}^K \biggl( (1 + e) \| \mathbf{s}_{k - 1}' - \mathbf{s}_{k - 1}'' \|_2 + \biggl( e \| \mathbf{s}_{k - 1}'' \|_2 + (e - 1) D_b \biggr) \mathcal{B} \| \theta' - \theta'' \|_2 \biggr) \\
    &\leq L_2 \sum_{k = 1}^K \biggl( (1 + e) \| \mathbf{s}_{k - 1}' - \mathbf{s}_{k - 1}'' \|_2 \\
    &+ \biggl( e ( L_1^K (1 + e)^K + 1) \| \mathbf{s}_0 \|_2 + \gamma(K) \mathcal{A} + (e - 1) D_b \biggr) \mathcal{B} \| \theta' - \theta'' \|_2 \biggr) \\
    &\leq L_2 \sum_{k = 1}^K \biggl( (1 + e) \gamma (K) L_1 \biggl( e ( ( L_1^K (1 + e)^K + 1) \| \mathbf{s}_0 \|_2 + \gamma(K) \mathcal{A}) + (e - 1) D_b \biggr) \\
    &\times \mathcal{B} \| \theta' - \theta'' \|_2 + \biggl( e ( L_1^K (1 + e)^K + 1) \| \mathbf{s}_0 \|_2 + \gamma(K) \mathcal{A} + (e - 1) D_b \biggr) \mathcal{B} \| \theta' - \theta'' \|_2 \biggr) \\
    &= L_2 K ((1 + e) \gamma(K) L_1 + 1) \biggl( e ( L_1^K (1 + e)^K + 1) \| \mathbf{s}_0 \|_2 + \gamma(K) \mathcal{A} \\
    &+ (e - 1) D_b \biggr) \mathcal{B} \| \theta' - \theta'' \|_2
\end{align*}
\end{proof}

\begin{lem} \label{lemma:lpsgr}
If Assumptions \ref{as:lpsf}, \ref{as:lpssigma} are satisfied, then for any $\theta' = \{ \mathbf{\Omega}_1', \mathbf{\Omega}_2', \mathbf{b}', \mathbf{N}', \mathbf{Q}', \mathbf{W}_0' \} \in \mathbb{D}$ and $\theta'' = \{ \mathbf{\Omega}_1'', \mathbf{\Omega}_2'', \mathbf{b}'', \mathbf{N}'', \mathbf{Q}'', \mathbf{W}_0'' \} \in \mathbb{D}$ such that $\| \mathbf{N}' \|_2, \| \mathbf{Q}' \|_2, \| \mathbf{N}'' \|_2, \| \mathbf{Q}'' \|_2 \leq D, \| \mathbf{b}' \|_2, \| \mathbf{b}'' \|_2 \leq D_b$ for some $D, D_b > 0$ it holds that
\begin{equation*}
    \| \nabla F_\sigma(\theta') - \nabla F_\sigma (\theta'') \|_2 \leq \frac{\mathcal{C} \sqrt{l}}{\sigma} \| \theta' - \theta'' \|_2
\end{equation*}
where $\mathcal{C}$ is from the definition of Lemma \ref{lemma:lips}.
\end{lem}
\begin{proof}
We deduce that
\begin{align*}
    \| \nabla F_\sigma(\theta') - \nabla F_\sigma (\theta'') \|_2^2 &= \frac{1}{\sigma^2} \| \mathbb{E}_{\epsilon \sim \mathcal{N} (0, I)} ( F (\theta' + \sigma \epsilon) - F (\theta' + \sigma \epsilon) ) \epsilon \|_2^2 \\
    &\leq \frac{1}{\sigma^2} \mathbb{E}_{\epsilon \sim \mathcal{N} (0, I)} \| ( F (\theta' + \sigma \epsilon) - F (\theta' + \sigma \epsilon) ) \epsilon \|_2^2 \\
    &= \frac{1}{\sigma^2} \mathbb{E}_{\epsilon \sim \mathcal{N} (0, I)} (F (\theta' + \sigma \epsilon) - F (\theta' + \sigma \epsilon))^2 \| \epsilon \|_2^2 \\
    &\leq \frac{1}{\sigma^2} \mathbb{E}_{\epsilon \sim \mathcal{N} (0, I)} \mathcal{C}^2 \| \theta' - \theta'' \|_2^2 \| \epsilon \|_2^2 \\
    &= \frac{\mathcal{C}^2}{\sigma^2} \| \theta' - \theta'' \|_2^2 \cdot \mathbb{E}_{\epsilon \sim \mathcal{N} (0, I)} \| \epsilon \|_2^2 = \frac{\mathcal{C}^2 \cdot l}{\sigma^2} \| \theta' - \theta'' \|_2^2 .
\end{align*}
\end{proof}

\begin{lem} \label{lemma:bgrad}
If Assumption \ref{as:lpsf} is satisfied, then for any $\theta \in \mathbb{D}$
\begin{equation*}
    \mathbb{E} \| \widetilde{\nabla} F_\sigma (\theta) \|_2^2 \leq \frac{K^2 M^2 l}{\sigma^2}.
\end{equation*}
\end{lem}
\begin{proof}
By using Assumption \ref{as:lpsf} and that $| F (\theta) | = | \sum_{k = 1}^K l_k | \leq K M$ we derive
\begin{align*}
    ( \mathbb{E} \| \widetilde{\nabla} F_\sigma (\theta) \|_2 )^2 &\leq \mathbb{E} \| \widetilde{\nabla}_\theta F_\sigma (\theta) \|_2^2 = \frac{1}{v^2 \sigma^2} \mathbb{E}_{\{\epsilon_w \sim \mathcal{N} (0, I) \}} \| \sum_{w = 1}^v F(\theta + \sigma \epsilon_w) \epsilon_w \|_2^2 \\
    &\leq \frac{1}{v \sigma^2} \mathbb{E}_{\{\epsilon_w \sim \mathcal{N} (0, I) \}} \sum_{w = 1}^v \| F(\theta + \sigma \epsilon_w) \epsilon_w \|_2^2 \\
    &= \frac{1}{\sigma^2} \mathbb{E}_{\epsilon \sim \mathcal{N} (0, I)} \| F(\theta + \sigma \epsilon) \epsilon \|_2^2 = \frac{1}{\sigma^2} \mathbb{E}_{\epsilon \sim \mathcal{N} (0, I)} F(\theta + \sigma \epsilon)^2 \| \epsilon \|_2^2 \\
    &\leq \frac{1}{\sigma^2} \mathbb{E}_{\epsilon \sim \mathcal{N} (0, I)} K^2 M^2 \| \epsilon \|_2^2 = \frac{K^2 M^2}{\sigma^2} \mathbb{E}_{\epsilon \sim \mathcal{N} (0, I)} \| \epsilon \|_2^2 = \frac{K^2 M^2 l}{\sigma^2}
\end{align*}
\end{proof}

\begin{proof}[Proof of Theorem \ref{th:conv}]

Hereafter we assume that $\mathcal{F}_{\tau,D,D_b}$ holds for random iterates $\theta^{(0)}, \dots, \theta^{(\tau)}$. According to Lemma \ref{lemma:lpsgr}, $\frac{\mathcal{C} \sqrt{l}}{\sigma}$ is a bound on $F_\sigma (\theta)$'s Hessian on $\{ \theta \in \mathbb{D} \, | \, \| \mathbf{N} \|_2, \| \mathbf{Q} \|_2 < D, \mathbf{b} < D_b \}$. We apply \cite[Appendix]{srgd} to derive that for any $\tau' \leq \tau$
\begin{equation} \label{eq:bnd2}
    F(\theta^{(\tau')}) - F(\theta^{(\tau' - 1)}) \geq \alpha_{\tau'} \nabla_\mathcal{R} F_\sigma (\theta^{(\tau' - 1)})^\top \widetilde{\nabla}_\mathcal{R} F_\sigma (\theta^{(\tau' - 1)}) - \frac{\mathcal{C} \sqrt{l}}{2 \sigma} \alpha_{\tau'}^2 \| \widetilde{\nabla}_\mathcal{R} F_\sigma (\theta^{(\tau' - 1)}) \|_F^2
\end{equation}
Let $\mathcal{A}$ denote a sigma-algebra associated with $\theta^{(1)}, \dots, \theta^{(\tau' - 1)}$. We use that $\mathbb{E}[\widetilde{\nabla}_\mathcal{R} F_\sigma (\theta^{(\tau' - 1)}) | \mathcal{A}] = \nabla_\mathcal{R} F_\sigma (\theta^{(\tau' - 1)})$ and $\mathbb{E}[ \| \widetilde{\nabla}_\mathcal{R} F_\sigma (\theta^{(\tau' - 1)}) \|_2^2 | \mathcal{A}] \leq \frac{K^2 M^2 l}{\sigma^2}$ (Lemma \ref{lemma:bgrad}) and take an expectation of (\ref{eq:bnd1}) conditioned on $\mathcal{A}$:
\begin{equation*}
    \mathbb{E} [F(\theta^{(\tau')}) | \mathcal{A} ] - F(\theta^{(\tau' - 1)}) \geq \alpha_{\tau'} \| \nabla_\mathcal{R} F_\sigma (\theta^{(\tau' - 1)}) \|_2^2 - \frac{\mathcal{C} \sqrt{l}}{2 \sigma} \alpha_{\tau'}^2 \cdot \frac{K^2 M^2 l}{\sigma^2}.
\end{equation*}
Regroup the last inequality and take a full expectation to obtain
\begin{equation*}
    \alpha_{\tau'} \mathbb{E} \| \nabla_\mathcal{R} F_\sigma (\theta^{(\tau' - 1)}) \|_2^2 \leq \mathbb{E} F(\theta^{(\tau')}) - \mathbb{E} F(\theta^{(\tau' - 1)}) + \frac{\mathcal{C} \sqrt{l}}{2 \sigma} \alpha_{\tau'}^2 \cdot \frac{K^2 M^2 l}{\sigma^2}.
\end{equation*}
Perform a summation of the last inequality for all $\tau' \in \{ 1, \dots, \tau \}$:
\begin{align*}
    \sum_{\tau' = 1}^\tau \alpha_{\tau'} \mathbb{E} \| \nabla_\mathcal{R} F_\sigma (\theta^{(\tau' - 1)}) \|_2^2 &\leq \mathbb{E} F(\theta^{(\tau)}) - F(\theta^{(0)}) +  \frac{\mathcal{C} K^2 M^2l^{\frac32}}{2 \sigma^3} \sum_{\tau' = 1}^\tau \alpha_{\tau'}^2 \\
    &\leq K M - F(\theta^{(0)}) +  \frac{\mathcal{C} K^2 M^2l^{\frac32}}{2 \sigma^3} \sum_{\tau' = 1}^\tau \alpha_{\tau'}^2
\end{align*}
since $F(\theta) \leq K M$ for all $\theta \in \mathbb{D}$. We next observe that
\begin{align*}
    \min_{0 \leq \tau' < \tau} \mathbb{E} \| \nabla_\mathcal{R} F_\sigma (\theta^{(\tau')}) \|_2^2 &= \frac{\sum_{\tau' = 1}^\tau \alpha_{\tau'} \min_{0 < \tau' \leq \tau} \mathbb{E} \| \nabla_\mathcal{R} F_\sigma (\theta^{(\tau' - 1)}) \|_2^2}{\sum_{\tau' = 1}^\tau \alpha_{\tau'}} \\
    &\leq \frac{1}{\sum_{\tau' = 1}^\tau \alpha_{\tau'}} ( K M - F(\theta^{(0)})) +  \frac{\sum_{\tau' = 1}^\tau \alpha_{\tau'}^2}{\sum_{\tau' = 1}^\tau \alpha_{\tau'}} \cdot \frac{\mathcal{C} K^2 M^2l^{\frac32}}{2 \sigma^3} .
\end{align*}
The proof is concluded by observing that $\sum_{\tau' = 1}^\tau \alpha_{\tau'} = \Omega (\tau^{0.5})$ and $\sum_{\tau' = 1}^\tau \alpha_{\tau'}^2 = O (\log \tau) = o(\tau^\epsilon)$ for any $\epsilon > 0$.

\end{proof}

\end{document}